\def\year{2021}\relax
\documentclass[letterpaper]{article} 
\usepackage{aaai21}  
\usepackage{times}  
\usepackage{helvet} 
\usepackage{courier}  
\usepackage[hyphens]{url}  
\usepackage{graphicx} 
\usepackage{natbib}  
\usepackage{caption} 

\usepackage{multirow}
\usepackage{bm}
\usepackage{algorithm}
\usepackage{algorithmic}
\usepackage{booktabs}
\usepackage{amsmath}
\usepackage{amsthm}
\usepackage{amssymb}
\usepackage{mathtools}
\usepackage{tikz}
\usepackage{xcolor}
\usetikzlibrary{arrows}
\usepackage{nicefrac} 

\usepackage{algorithm}
\usepackage{algorithmic}
\usepackage{bm}
\usepackage{subcaption}

\newcommand{\vect}[1]{\mathbf{#1}}

\DeclareMathOperator*{\argmax}{arg\,max}

\usepackage{listings}
\newtheorem{thm}{Theorem}[]
\newtheorem{thmA}{Theorem}[]
\newtheorem{lem}{Lemma}[]
\newtheorem{cor}{Corollary}[]
\newtheorem{prop}{Proposition}[]

\newtheorem{defn}{Definition}[]

\newcommand{\PH}{\phantom{0}}

\newcommand{\cc}[1]{#1}

\title{Exploring the Vulnerability of Deep Neural Networks: \\ A Study of Parameter Corruption}

\author{
Xu Sun,\textsuperscript{\rm 1,2}{\thanks{Equal Contribution.}}
Zhiyuan Zhang,\textsuperscript{\rm 1}$^*$
Xuancheng Ren,\textsuperscript{\rm 1}
Ruixuan Luo,\textsuperscript{\rm 2}
Liangyou Li\textsuperscript{\rm 3}\\
}
\affiliations{
\textsuperscript{\rm 1}MOE Key Laboratory of Computational Linguistic, School of EECS, Peking University\\
\textsuperscript{\rm 2}Center for Data Science, Peking University\\
\textsuperscript{\rm 3}Huawei Noah’s Ark Lab\\
\{xusun, zzy1210,renxc,luoruixuan97\}@pku.edu.cn, liliangyou@huawei.com
}

\begin{document}

\maketitle

\begin{abstract}
We argue that the vulnerability of model parameters is of crucial value to the study of model robustness and generalization but little research has been devoted to understanding this matter.
In this work, we propose an indicator to measure the robustness of neural network parameters by exploiting their vulnerability via parameter corruption. The proposed indicator describes the maximum loss variation in the non-trivial worst-case scenario under parameter corruption. For practical purposes, we give a gradient-based estimation, which is far more effective than random corruption trials that can hardly induce the worst accuracy degradation. Equipped with theoretical support and empirical validation, we are able to systematically investigate the robustness of different model parameters and reveal vulnerability of deep neural networks that has been rarely paid attention to before. Moreover, we can enhance the models accordingly with the proposed adversarial corruption-resistant training, which not only improves the parameter robustness but also translates into accuracy elevation.
\end{abstract}

\section{Introduction}

Despite the promising performance of Deep neural networks (DNNs), research has discovered that DNNs are vulnerable to adversarial examples, i.e., simple perturbations to input data can mislead models~\citep{DBLP:journals/corr/GoodfellowSS14,DBLP:conf/iclr/KurakinGB17a,DBLP:conf/iclr/MadryMSTV18}. 
These findings concern the vulnerability of DNNs against input data. However, the vulnerability of DNNs does not only exhibit in input data. As functions of both input data and model parameters, the parameters of neural networks are a source of vulnerability of equal importance. For neural networks deployed on electronic computers, parameter attacks can be conducted in the form of training data poisoning~\citep{backdoor1,backdoor2,backdoors3}, bit flipping~\citep{TBT}, compression~\citep{Stronger} or quantization~\citep{Data-Free-Quant,model_robustness}. For neural networks deployed in physical devices, advances in hardware neural networks \cite{optical-NN, misra2010artificial,abdelsalam2018efficient,salimi2019digital,weber2019amplifier,bui2019scalable} also call for study in parameter vulnerability because of hardware deterioration and background noise, which can be seen as parameter corruption. More importantly, study on parameter vulnerability can deepen our understanding of various mechanisms in neural networks, inspiring innovation in architecture design and training paradigm.

\begin{figure}[t]
\centering
\includegraphics[height=8\baselineskip]{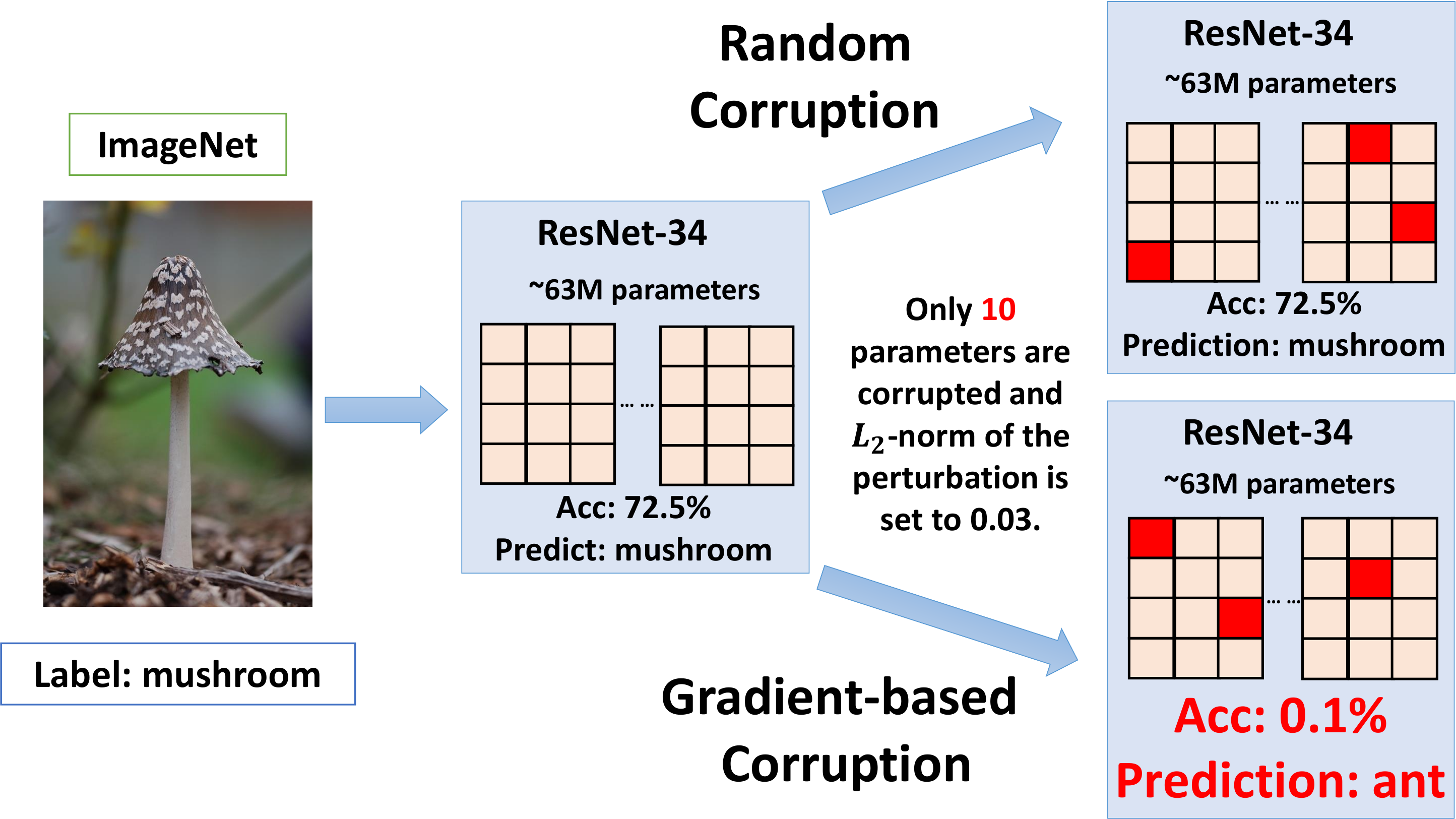}
\caption{Parameter corruptions with ResNet-34 on ImageNet. It shows that deep neural networks are robust to random corruptions, but the accuracy can drop significantly in the worst case suggested by the gradient-based method. {The accuracy is measured on the development set.}}
\label{fig:error}
\end{figure}

To probe the vulnerability of neural network parameters and evaluate the \textit{parameter robustness},
we propose an indicator that measures the maximum loss change caused by small perturbations on model parameters in the non-trivial worst-cased scenario. The perturbations can be seen as artificial parameter corruptions. We give an infinitesimal gradient-based estimation of the indicator that is efficient for practical purposes compared with random corruption trials, which can hardly induce optimal loss degradation. Our theoretical and empirical results both validate the effectiveness of the proposed gradient-based method. As shown in Figure~\ref{fig:error}, model parameters are generally resistant to random corruptions but the worst outlook can be quite bleak suggested by the gradient-based corruption result. 

Intuitively, the indicator shows the maximum altitude ascent within a certain distance of the current parameter on the loss surface, as illustrated conceptually in Figure~\ref{fig:loss_curve}. {The traditional learning algorithms focus on obtaining lower loss, which means generally the parameters at point B are preferred. However, the local geometry of the landscape also indicates the generalization performance of the learning algorithm}~\citep{sharp-local1,sharp-local2}. {The parameters at point B demonstrate critical vulnerability to parameter corruptions, while the parameters at point A are a better choice since larger perturbations are required to render significant loss change. It is also  observed in our experiments that the parameters at point A have better generalization performance as a result of corruption-resistance. 
} 

\begin{figure}[t]
\centering
\includegraphics[height=8\baselineskip]{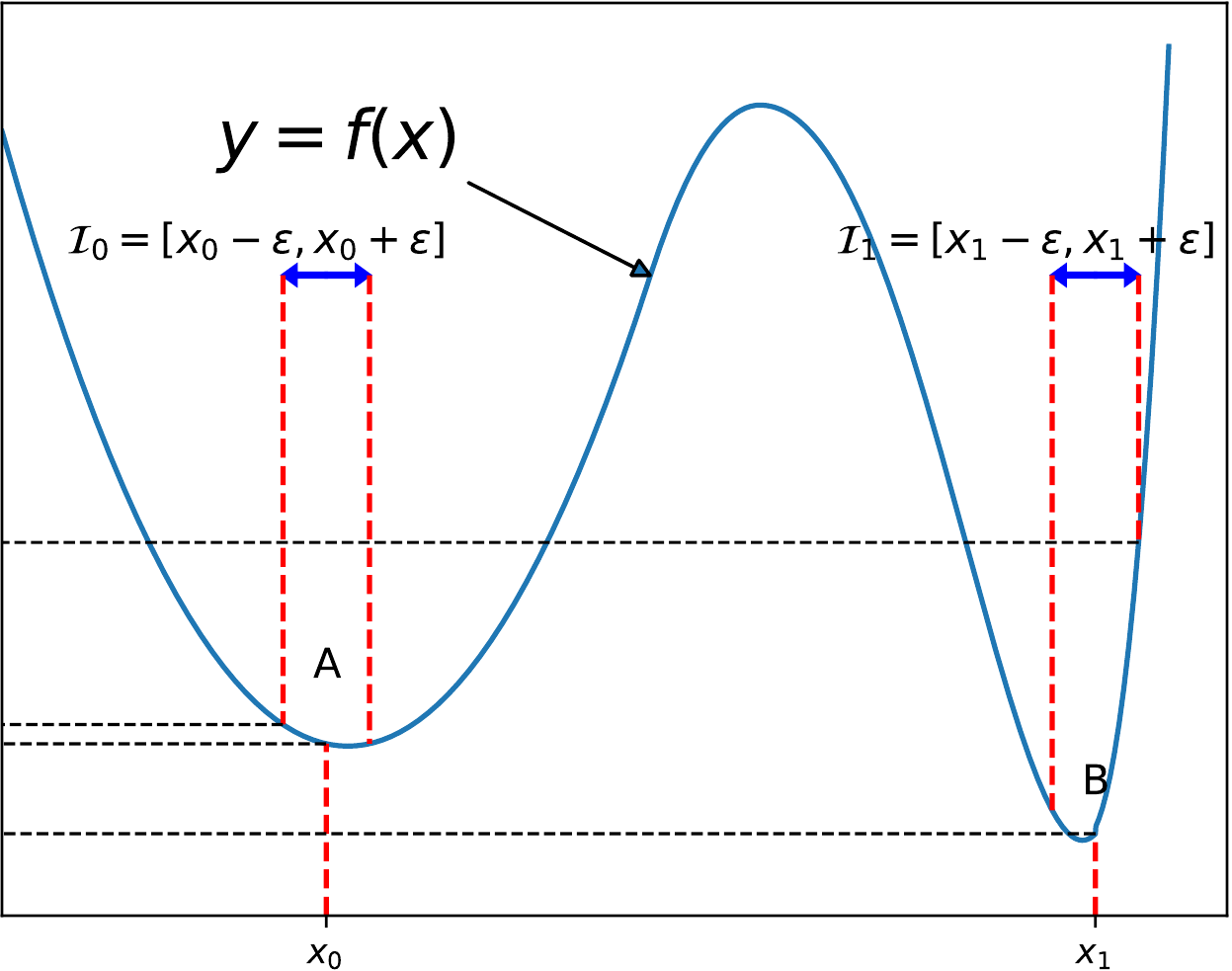}
\caption{In this illustration of the loss function, traditional optimizer prefers $B$ with the lower loss rather than $A$, because $B$ has the lower loss. However, parameters at point $B$ are more vulnerable to parameter corruption, as $\max_{x\in\mathcal{I}_0}(f(x)-f(x_0))< \max_{x\in\mathcal{I}_1}(f(x)-f(x_1))$. Based on our experiments, we argue that parameters that are resistant to corruption, e.g., at point $A$, can embody potentially better robustness and generalization.}
\label{fig:loss_curve}
\end{figure}

Equipped with the proposed indicator, we are able to systematically analyze the parameter robustness and probe the vulnerability of different components in a deep neural network via observing the accuracy degradation after applying  corruptions to its parameters. 
%
Furthermore, the comparisons between the gradient-based and the random corruption for estimating the indicator suggest that the neighborhood of the learned parameters on the loss surface is generally flattish except for certain steep directions. If we can push the parameters away from the steep directions, the robustness of the parameters can be improved significantly. Therefore, we propose to conduct adversarial corruption-resistant training that incorporates virtual parameter corruptions to find parameters without steep directions in the neighborhood. Experimental results show that the proposed method not only improves the parameter robustness of deep neural networks but also elevates their accuracy in application tasks.

Our main contributions are as follows:
\begin{itemize}

\item To understand the parameter vulnerability of deep neural networks, which is fundamentally related to model robustness and generalization, we propose an indicator that measures the maximum loss change if small perturbations are applied on parameters, i.e., parameter corruptions. The proposed gradient-based estimation is far more effective in exposing the vulnerability than random corruption trials, validated by both theoretical and empirical results.

\item The indicator is used to probe the vulnerability of different kinds of parameters with diverse structural characteristics in a trained neural network. Through systematic analyses of representative architectures, we summarize divergent vulnerability of neural network parameters, especially bringing attention to normalization layers.

\item To improve the robustness of the models with respect to parameters, we propose to enhance the training of deep neural networks by taking the parameter vulnerability into account and introduce the adversarial corruption-resistant training that can improve the accuracy and the generalization performance of deep neural networks.

\end{itemize}

\begin{figure*}[!t]
\begin{minipage}[!t]{0.48\textwidth}
\begin{algorithm}[H]
\small
   \caption{Random Corruption}
   \label{alg:random}
\begin{algorithmic}[1]
   \REQUIRE Parameter vector $\vect{w}\in\mathbb{R}^k$, set of corruption constraints $S$
    \STATE Sample $\vect{r}\sim N(0, 1)$
    \STATE Solve the random corruption $\vect{\tilde a}$ according to Eq.(\ref{equ:random})
    \STATE Update the parameter vector $\vect{w}\gets\vect{w}+\vect{\tilde a}$
\end{algorithmic}
\end{algorithm}
\end{minipage}
\hfill
\begin{minipage}[!t]{0.48\textwidth}
\begin{algorithm}[H]
\small
\caption{Gradient-Based Corruption}
   \label{alg:gradient}
\begin{algorithmic}[1]
    \REQUIRE Parameter vector $\vect{w}\in\mathbb{R}^k$, set of corruption constraints  $S$, loss function $\mathcal{L}$ and dataset $\mathcal{D}$
    \STATE Obtain the gradient $\vect{g}\gets \nicefrac{\partial \mathcal{L}(\vect{w};\mathcal{D})}{\partial \vect{w}}$
    \STATE Solve the corruption $\vect{\hat a}$ in Eq.(\ref{equ:proposed}) with $S$ and $\vect{g}$
    \STATE Update the parameter vector $\vect{w}\gets\vect{w}+\vect{\hat a}$
\end{algorithmic}
\end{algorithm}
\end{minipage}
\end{figure*}

\section{Parameter Corruption}
\label{sec:method}

In this section, we introduce the problem of parameter corruption and the proposed indicator. Then, we describe the Monte-Carlo estimation and the gradient-based estimation of the indicator backed with theoretical support.



Before delving into the specifics, we first introduce our notations. Let $\mathcal{N}$ denote a neural network, $\vect{w}\in\mathbb{R}^k$ denote a $k$-dimensional subspace of its parameter space, and $\mathcal{L}(\vect{w}; \mathcal{D})$ denote the loss function of $\mathcal{N}$ on the dataset $\mathcal{D}$, regarding to the specific parameter subspace $\vect{w}$. Taking a $k$-dimensional subspace allows a more general analysis on a specific group of parameters.

{To expose the vulnerability of model parameters, we propose to adopt the approach of parameter corruption. 
} 
To formally analyze its effect on neural networks and eliminate trivial corruption, we formulate the parameter corruption as a small perturbation $\vect{a}\in\mathbb{R}^k$ to the parameter vector $\vect{w}$. The corrupted parameter vector becomes $\vect{w}+\vect{a}$.  The small perturbation requirement is realized as a constraint set of the parameter corruptions.
\begin{defn}[Corruption Constraint]
The corruption constraint is specified by the set
\begin{equation}
S=\{\vect{a}:\|\vect{a}\|_p= \epsilon \text{ and }\|\vect{a}\|_0\le n\},
\end{equation}
where $\|\cdot\|_0$ denotes the number of non-zero elements in a vector and $1\le n\le k$ denotes the maximum number of corrupted parameters. $\epsilon$ is a small positive real number and $\|\cdot\|_p$ denotes the $L_p$-norm where $p\ge 1$ such that $\|\cdot\|_p$ is a valid distance in Euclidean geometry.
\end{defn}
For example, the set $S=\{\vect{a}:\|\vect{a}\|_2= \epsilon\}$ specifies that the corruption should be on a hypersphere with a radius of $\epsilon$ and no limit on the number of corrupted parameters.

Suppose $\Delta\mathcal{L}(\vect{w}, \vect{a}; \mathcal{D})=\mathcal{L}(\vect{w}+\vect{a}; \mathcal{D})-\mathcal{L}(\vect{w}; \mathcal{D})$ denotes the loss change. To evaluate the effect of parameter corruption, it is most reasonable to consider the worst-case scenario and thus, we propose the indicator as the maximum loss change under the corruption constraints. The optimal parameter corruption is defined accordingly.

\begin{defn}[Indicator and Optimal Parameter Corruption]
The indicator $\Delta_\text{max}\mathcal{L}(\vect{w}, S, \mathcal{D})$ and the optimal parameter corruption $\vect{a}^*$ are defined as:
\begin{align}
\Delta_\text{max}\mathcal{L}(\vect{w}, S, \mathcal{D})=\max\limits_{\vect{a}\in S}\Delta\mathcal{L}(\vect{w}, \vect{a}, \mathcal{D}) 
 ,\\\vect{a}^*=\argmax_{\vect{a}\in S}\Delta\mathcal{L}(\vect{w}, \vect{a}, \mathcal{D}).
\end{align}
\end{defn}

{Let $\vect{g}$ denote $\nicefrac{\partial\mathcal{L}(\vect{w}; \mathcal{D})}{\partial \vect{w}}$ and $\textbf{H}$ denote the Hessian matrix; suppose $\|\vect{g}\|_2=G>0$. Using the second-order Taylor expansion, we estimate the loss change and the indicator}:
\begin{equation}
\Delta\mathcal{L}(\vect{w},\vect{a}; \mathcal{D}) = \vect{a}^\text{T}\vect{g}+\frac{1}{2}\vect{a}^\text{T}\textbf{H}\vect{a}+o(\epsilon^2)=f(\vect{a})+o(\epsilon).
\end{equation}
{Here, $f(\vect{a})=\vect{a}^\text{T}\vect{g}$ is a first-order estimation of $\Delta\mathcal{L}(\vect{w}, \vect{a}; \mathcal{D})$ and meanwhile the inner product of the parameter corruption $\vect{a}$ and the gradient $\vect{g}$, based on which we maximize the alternative inner
product instead of initial loss function to estimate the indicator.}

We provide and analyze two methods to understand the effect of parameter corruption, which estimate the value of the indicator based on constructive, artificial, theoretical parameter corruptions. Comparing the two methods, the random parameter corruption gives a Monte-Carlo estimation of the indicator and the gradient-based parameter corruption gives an infinitesimal estimation that can effectively capture the worst case.
Please refer to Appendix for detailed proofs of propositions and theorems.  

\subsection{Random Corruption}
\label{sec:random}

We first analyze the random case.
As we know, randomly sampling a perturbation vector $\vect{a}$  does not necessarily conform to the constraint set $S$ and it is complex to generate corruption uniformly distributed in $S$ as the generation is determined by the shape of $S$ and is not universal enough. To eliminate the problem, we define the random parameter corruptions used in this estimation as maximizing an alternative inner product $\vect{a}^\text{T}\vect{r}$ under the constraint, based on a random vector $\vect{r}$ instead of the gradient $\vect{g}$ to ensure the randomness.

\begin{defn}[Random Parameter Corruption and Monte-Carlo Estimation]
\label{def:1}
Given a randomly sampled vector $\vect{r}\sim N(0, 1)$, a valid random corruption $\vect{\tilde a}$ for a Monte-Carlo estimation of the indicator in the constraint set $S$, which has a closed-form solution, is
\begin{equation}
\vect{\tilde a}=\argmax_{\vect{a}\in S}\vect{a}^\text{T}\vect{r}=\epsilon \left(\text{sgn}(\vect{h})\odot\frac{|\vect{h}|^\frac{1}{p-1}}{\||\vect{h}|^\frac{1}{p-1}\|_p}\right)
\label{equ:random}
\end{equation}
where $\vect{h}=\text{top}_n(\vect{r})$. The $\text{top}_n(\vect{v})$ function retains top-$n$ magnitude of all $|\vect{v}|$ dimensions and set other dimensions to $0$, $\text{sgn}(\cdot)$ denotes the signum function, $|\cdot|$ denotes the point-wise absolute function, and $(\cdot)^\alpha$ denotes the point-wise $\alpha$-power function.
The loss change with the random corruption is a Monte-Carlo estimation of the indicator.
\end{defn}

The procedure to derive the random corruption vector under the Monte-Carlo estimation of the indicator is shown in Algorithm~\ref{alg:random}. The correctness and randomness of the resulting corruption vector are assured and the theoretical results are given in Appendix. Without losing generality, we discuss the characteristics of the loss change caused by random corruption under a representative corruption constraint in Theorem~\ref{thm:random}. The proof and further analysis are in Appendix.

\begin{thm}[Distribution of Random Corruption]
\label{thm:random}
Given the constraint set  $S=\{\vect{a}:\|\vect{a}\|_2= \epsilon\}$ and a generated random corruption $\vect{\tilde a}$ by Eq.(\ref{equ:random}), which in turn obeys a uniform distribution on $\|\vect{\tilde a}\|_2=\epsilon$. The first-order estimation of $\Delta_\text{max}\mathcal{L}(\vect{w}, S, \mathcal{D})$ and the expectation of the loss change caused by random corruption is
\begin{align}
\Delta_\text{max}\mathcal{L}(\vect{w}, S, \mathcal{D}) &= \epsilon G+o(\epsilon);\\
\mathbb{E}_{\|\vect{\tilde a}\|_2 = \epsilon}[\Delta\mathcal{L}(\vect{w}, \vect{\tilde a}; \mathcal{D})]&=O\left(\frac{tr(\textbf{H})}{k}\epsilon^2\right).
\end{align}

Define $\eta=\nicefrac{|\vect{\tilde a}^\text{T}\vect{g}|}{\epsilon G}$ and $\eta\in [0, 1]$, which is an estimation of  $\nicefrac{|\Delta\mathcal{L}(\vect{w}, \vect{\tilde a}, \mathcal{D})| }{\Delta_\text{max}\mathcal{L}(\vect{w}, S, \mathcal{D})}$, then the probability density function $p_\eta(x)$ of $\eta$ and the cumulative density $P(\eta \le x)$ function of $\eta$ are
\begin{align}
p_{\eta}(x)&=\frac{2\Gamma(\frac{k}{2})}{\sqrt{\pi}\Gamma(\frac{k-1}{2})}(1-x^2)^{\frac{k-3}{2}};  \label{equ:random_destiny1}\\
P(\eta \le x)&=\frac{2xF_1(\frac{1}{2}, \frac{3-k}{2};\frac{3}{2}; x^2)}{B(\frac{k-1}{2}, \frac{1}{2})};
\end{align}
where $k$ denotes the number of corrupted parameters and $\Gamma(\cdot)$, $B(\cdot, \cdot)$ and $F_1(\cdot, \cdot;  \cdot; \cdot)$ denote the gamma function, beta function and hyper-geometric function, respectively.
\end{thm}

Theorem~\ref{thm:random} states that the expectation of loss change of random corruption is an infinitesimal of higher order compared to $\Delta_\text{max}\mathcal{L}(\vect{w}, S, \mathcal{D})$ when $\epsilon\to 0$. In addition, it is unlikely for multiple random trials to induce the optimal loss change corresponding to the indicator. For a deep neural network, the number of corrupted parameters can be considerably large. According to Eq.(\ref{equ:random_destiny1}), $\eta$ will be concentrated near $0$. Thus, theoretically, it is not generally possible for the random corruption to cause substantial loss changes in this circumstance, making it ineffective in finding vulnerability.


\begin{table*}[ht]

\centering

\setlength{\tabcolsep}{2.4pt}
\scriptsize
\begin{tabular}{@{}lcccccccccc@{}}
\toprule

  \bf Dataset & \multicolumn{2}{c}{\textbf{ImageNet (Acc $\uparrow$)} } &  \multicolumn{2}{c}{\textbf{CIFAR-10 (Acc $\uparrow$)}} & \multicolumn{2}{c}{\textbf{PTB-LM (Log PPL $\downarrow$)}} & \multicolumn{2}{c}{\textbf{PTB-Parsing (UAS $\uparrow$)} } & \multicolumn{2}{c}{\textbf{De-En (BLEU $\uparrow$)} } \\
 \midrule
  Base model & \multicolumn{4}{c}{{ResNet-34}} & \multicolumn{2}{c}{{LSTM}} & \multicolumn{2}{c}{{MLP}} & \multicolumn{2}{c}{{Transformer}} \\
 \midrule

 w/o corruption & \multicolumn{2}{c}{72.5 $\star$} & \multicolumn{2}{c}{94.3 $\star$} & \multicolumn{2}{c}{4.25 $\star$} & \multicolumn{2}{c}{87.3 $\star$} & \multicolumn{2}{c}{35.33 $\star$} \\

 \midrule
  Approach & Random & Proposed & Random & Proposed & Random & Proposed & Random & Proposed & Random & Proposed\\

\midrule
 {$n$=$k$, $\epsilon$=10\textsuperscript{-4}} & $\star$ & 62.2 (-10.3) & $\star$ & 93.3 (-1.0)\PH  & $\star$ & $\star$ & $\star$ & $\star$ & $\star$ & 35.21 (-0.12)\PH \\
 {$n$=$k$, $\epsilon$=10\textsuperscript{-3}} &  $\star$ & 22.2 (-50.3)  & $\star$  & 36.1 (-58.2)& $\star$ & $\star$ & $\star$ & 80.6 (-6.7)\PH &  $\star$ & 33.62 (-1.71)\PH  \\
 {$n$=$k$, $\epsilon$=10\textsuperscript{-2}} & 30.3 (-42.2) & \PH 0.1 (-72.4) & 75.1 (-19.2) &  10.0 (-84.3) & $\star$  & \PH 4.52 (+0.27)\PH & 79.8 (-7.5)\PH  & \PH 6.1 (-81.2) & 34.82 (-0.51)\PH & \PH 0.17 (-35.16)  \\
 {$n$=$k$, $\epsilon$=10\textsuperscript{-1}} & \PH 0.1 (-72.4) & \PH 0.1 (-72.4) & 10.0 (-84.3) & 10.0 (-84.3) & \PH 4.43 (+0.18)\PH  & 13.25 (+9.00)\PH  & \PH 0.0 (-87.3) & \PH 0.0 (-87.3) & \PH 0.00 (-35.33) & \PH 0.00 (-35.33) \\
 {$n$=$k$, $\epsilon$=1} & \PH 0.1 (-72.4) & \PH 0.1 (-72.4) & 10.0 (-84.3) & 10.0 (-84.3) & 32.21 (+27.96) & 48.92 (+44.67)  & \PH 0.0 (-87.3) & \PH 0.0 (-87.3) & \PH 0.00 (-35.33) & \PH 0.00 (-35.33) \\
 \midrule
 {$n$=100, $\epsilon$=10\textsuperscript{-2}} & $\star$ & $\star$ & $\star$ & $\star$ & $\star$ & $\star$  & $\star$ & 64.6 (-22.7) & $\star$ & $\star$ \\
 {$n$=100, $\epsilon$=10\textsuperscript{-1}} & $\star$ & 67.5 (-5.0)\PH & $\star$ & $\star$ & $\star$ & $\star$ &  $\star$ &  11.0 (-76.3) & $\star$ & 31.68 (-3.65)\PH   \\
 {$n$=100, $\epsilon$=1} & $\star$  & \PH 0.1 (-72.4) & $\star$ & $\star$ & $\star$ & $\star$ & 87.1 (-0.2)\PH & \PH 0.0 (-87.3) & 35.25 (-0.08)\PH & \PH 0.00 (-35.33)  \\
 {$n$=100, $\epsilon$=10\textsuperscript{1}} &\PH 0.2 (-72.3) & \PH 0.1 (-72.4) &  77.1 (-17.2) & 44.8 (-49.5) & $\star$ & $\star$ & 31.9 (-55.4) & \PH 0.0 (-87.3) & 11.58 (-23.75) & \PH 0.00 (-35.33)  \\
 {$n$=100, $\epsilon$=10\textsuperscript{2}} & \PH 0.1 (-72.4) & \PH 0.1 (-72.4) & 10.1 (-84.2) & \PH 9.6 (-84.7)  & 16.90 (+12.65) & 23.55 (+19.30) & \PH 0.0 (-87.3) & \PH 0.0 (-87.3) & \PH 0.00 (-35.33) & \PH 0.00 (-35.33) \\
\bottomrule
\end{tabular}
\caption{Comparisons of gradient-based corruption and random corruption under the corruption constraint ($L_{+\infty}$), with further study on the number $n$ of parameters to be corrupted. {Here, all parameters can be corrupted, that is, $k$ stands for the total number of model parameters and $n=k$ means the number of changed parameters is not limited.} $\uparrow$ means the higher value the better accuracy and $\downarrow$ means the opposite. $\star$ denotes original scores without parameter corruption and scores close to the original score (difference less than $0.1$).}
\label{tab:rand-results}
\end{table*}

\begin{figure*}[ht]
\centering

\subcaptionbox{ResNet-34 ($L_2$)}{\includegraphics[width=0.22\linewidth]{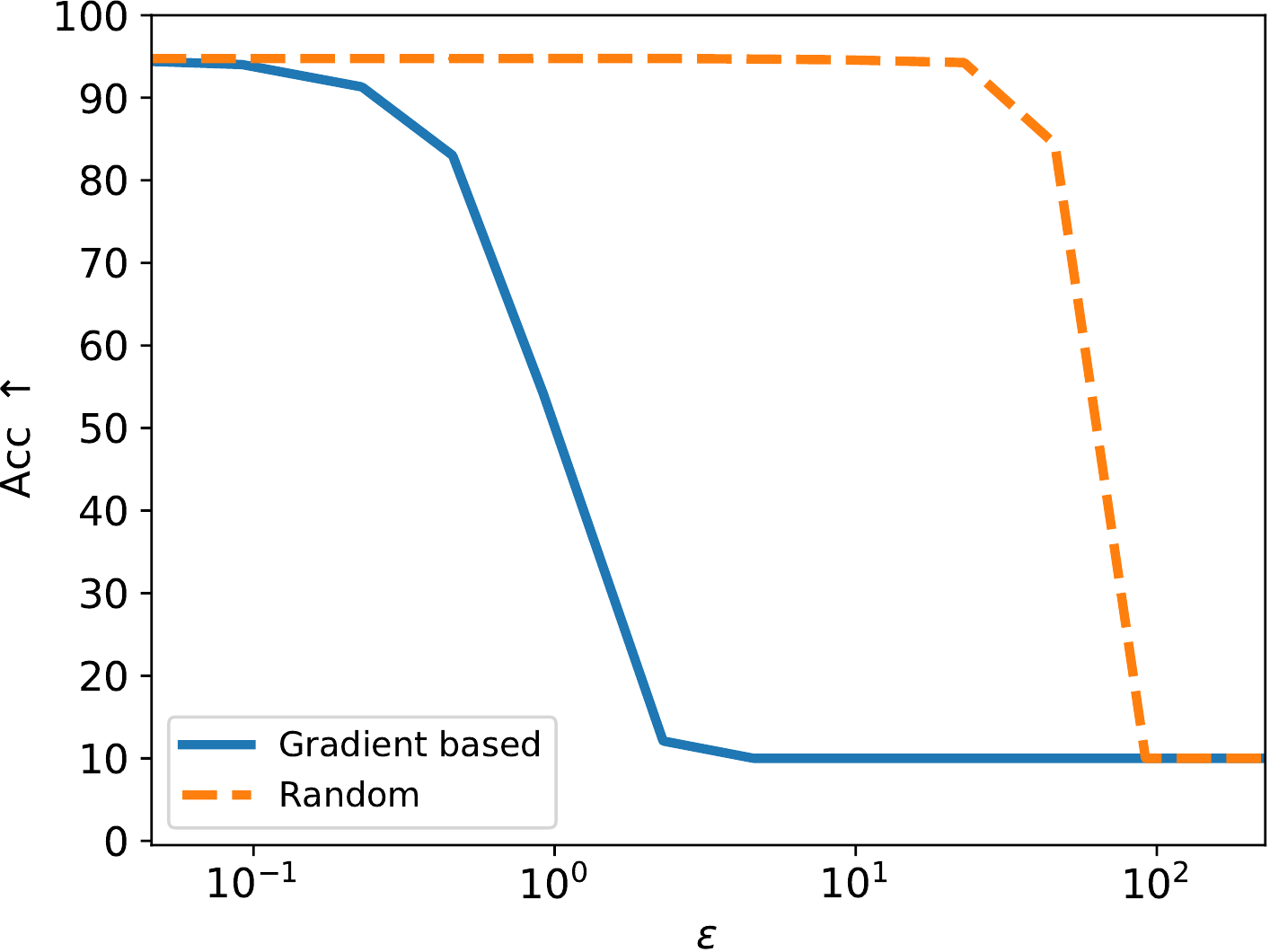}}
\hfil
\subcaptionbox{Transformer ($L_2$)}{\includegraphics[width=0.22\linewidth]{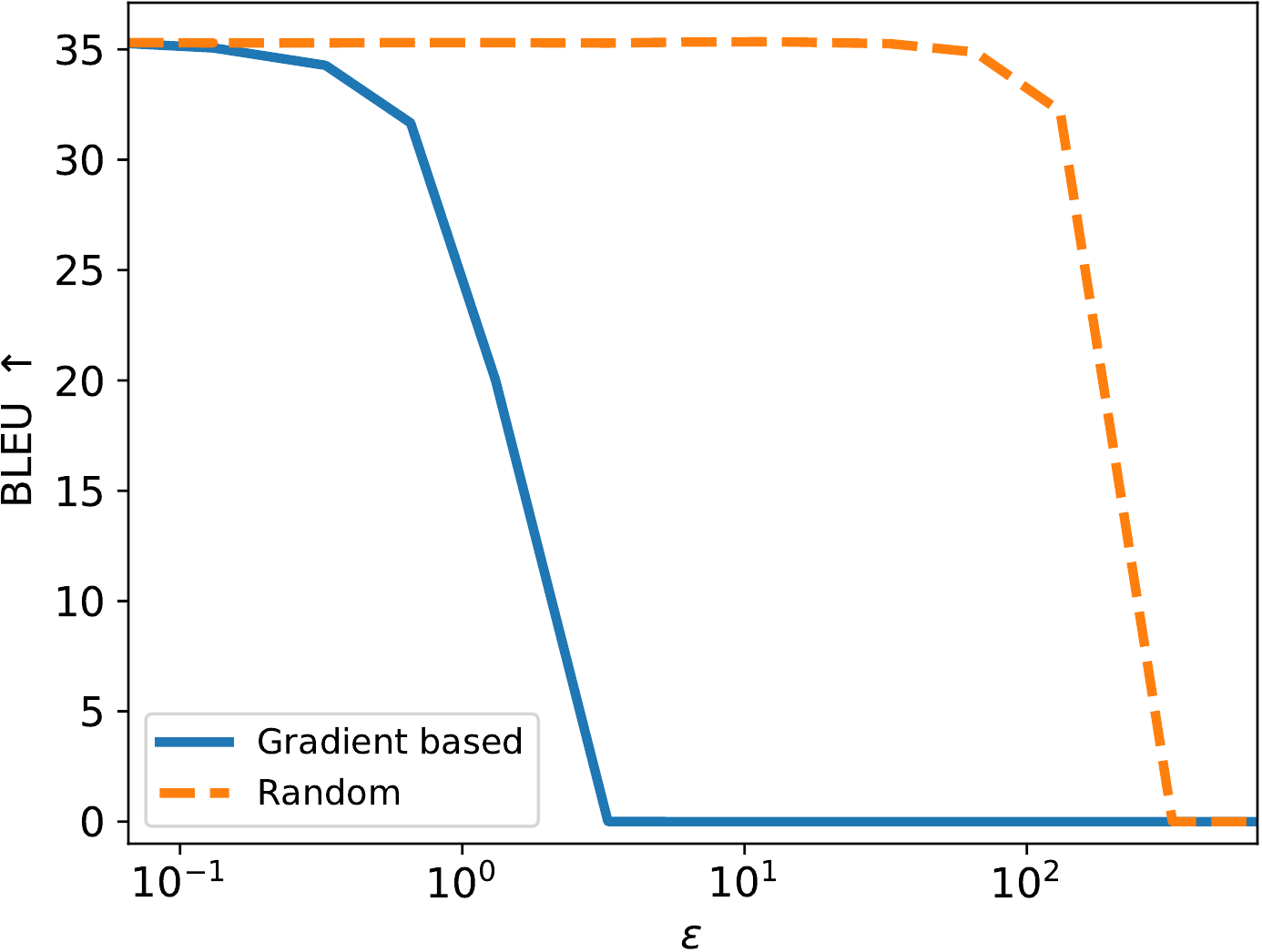}}
\hfil
\subcaptionbox{LSTM ($L_2$)}{\includegraphics[width=0.22\linewidth]{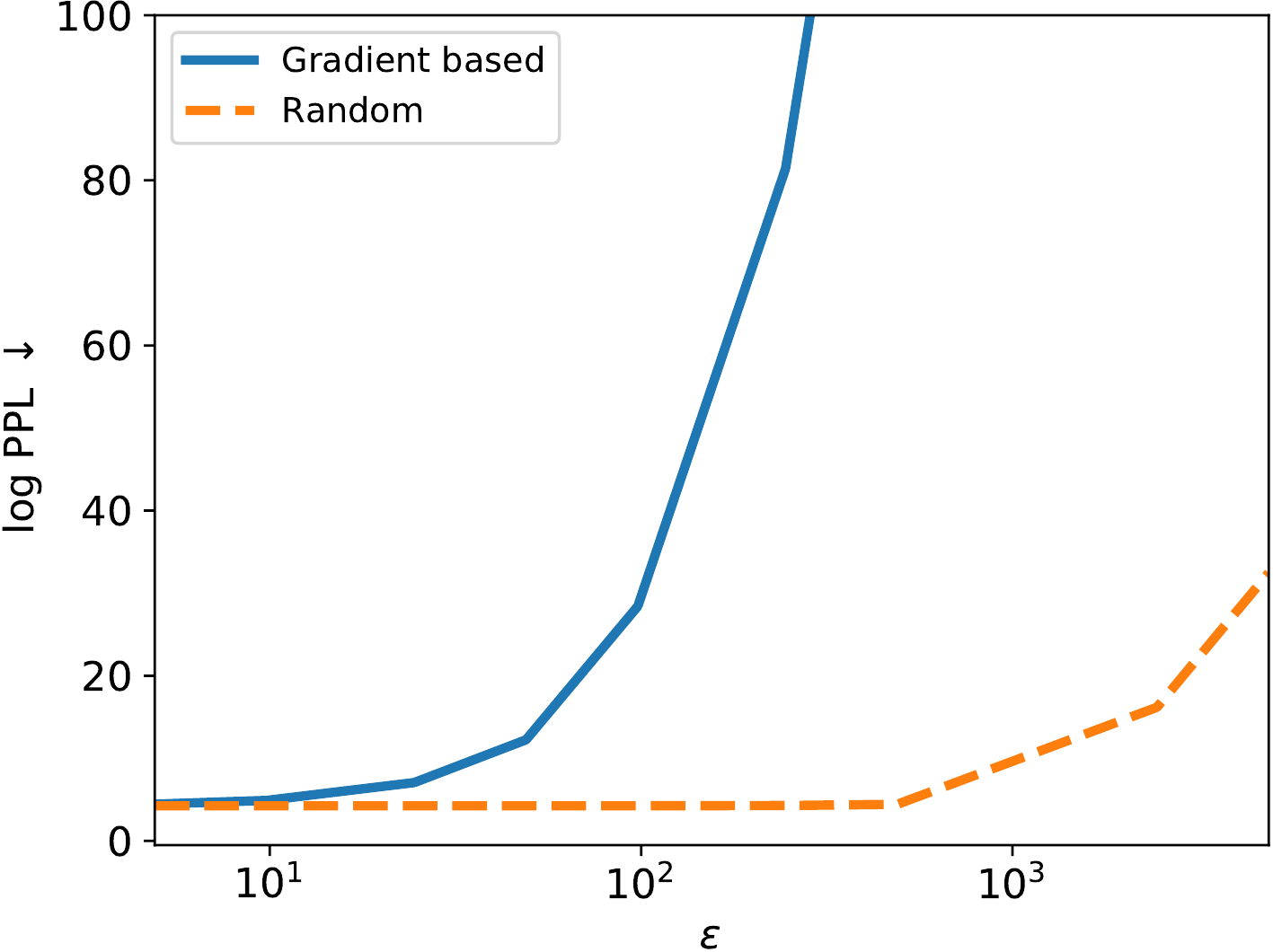}}
\hfil
\subcaptionbox{MLP ($L_2$)}{\includegraphics[width=0.22\linewidth]{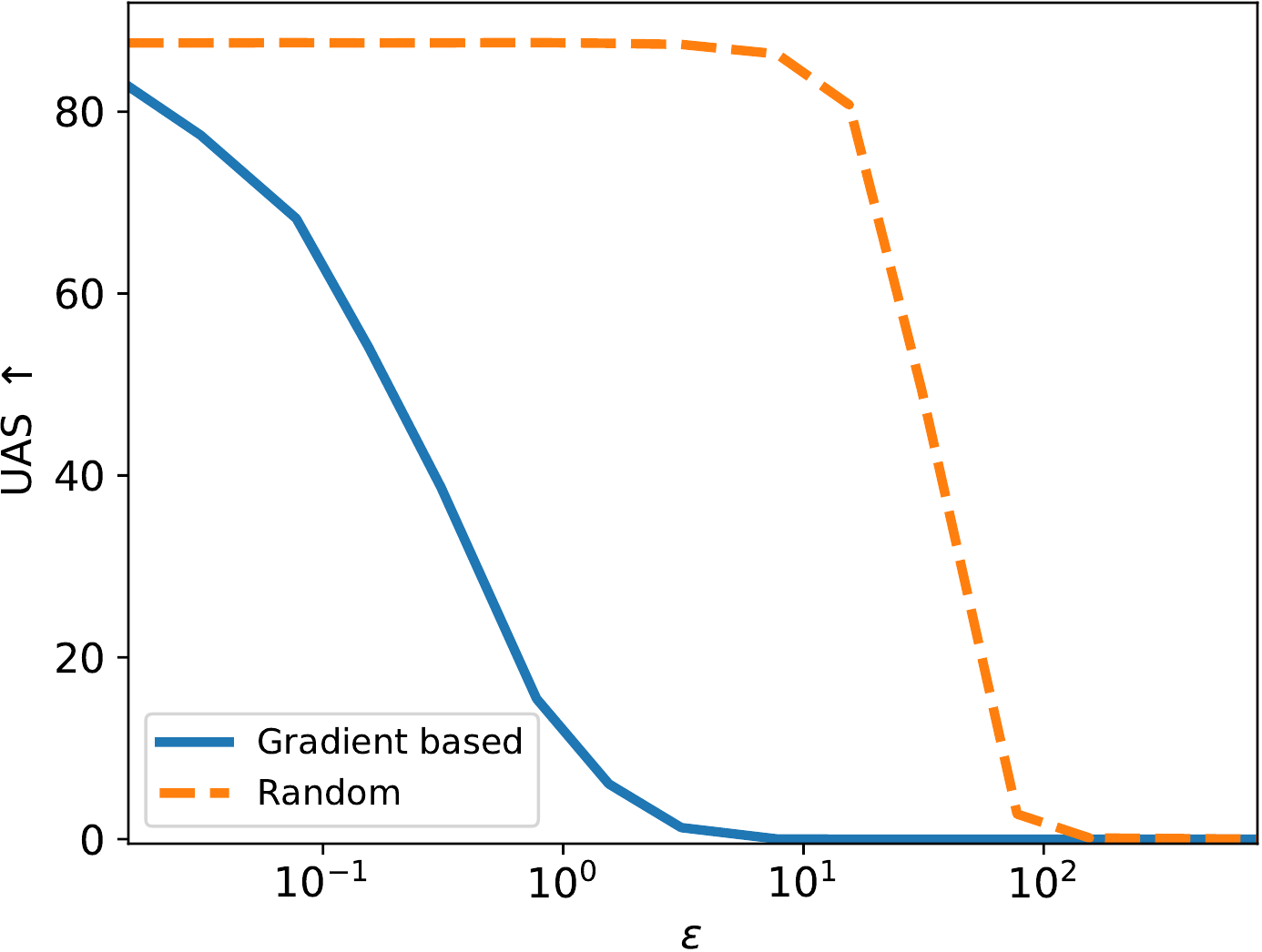}}
\subcaptionbox{ResNet-34 ($L_{+\infty}$)}{\includegraphics[width=0.22\linewidth]{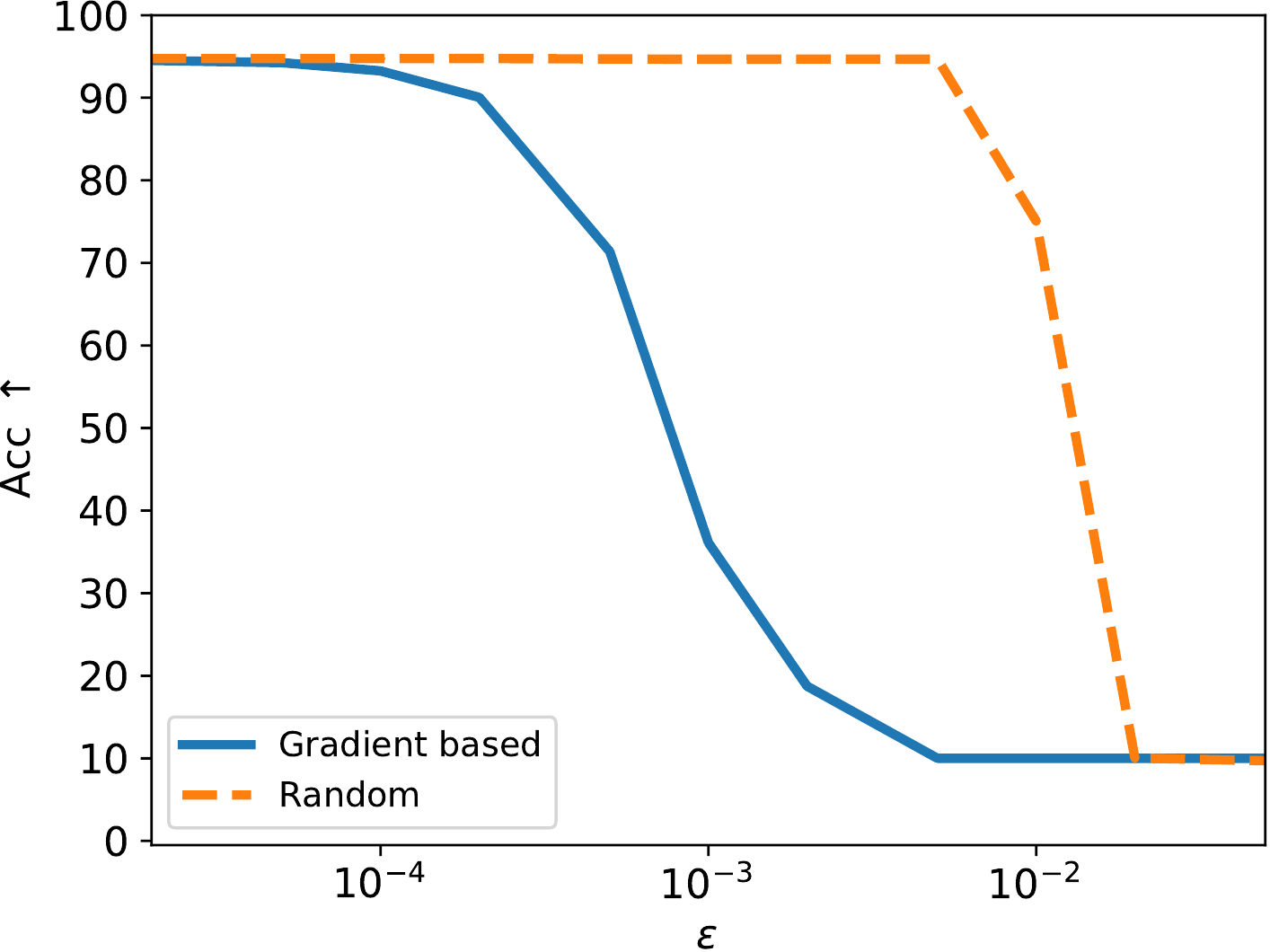}}
\hfil
\subcaptionbox{Transformer ($L_{+\infty}$)}{\includegraphics[width=0.22\linewidth]{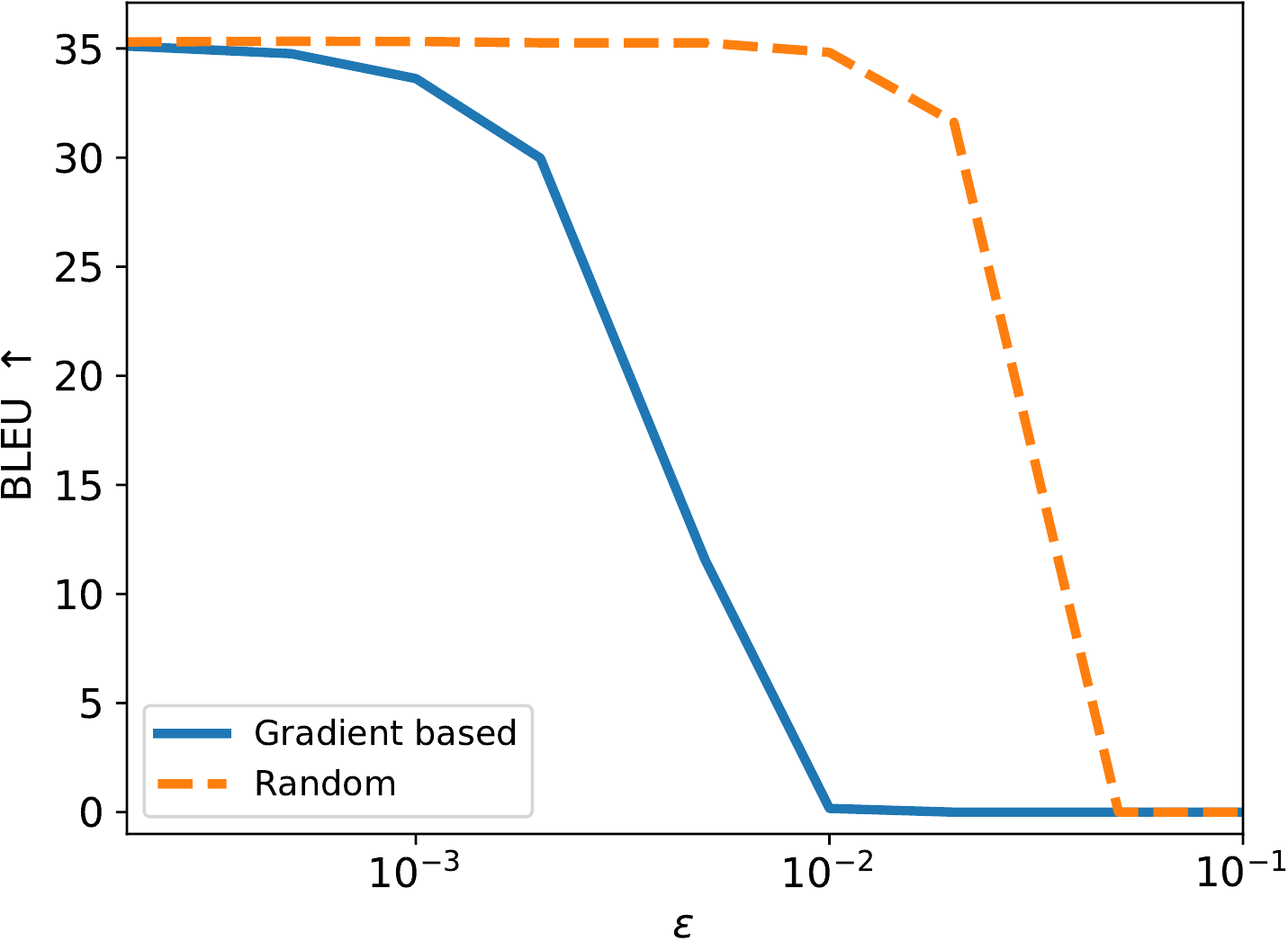}}
\hfil
\subcaptionbox{LSTM ($L_{+\infty}$)}{\includegraphics[width=0.22\linewidth]{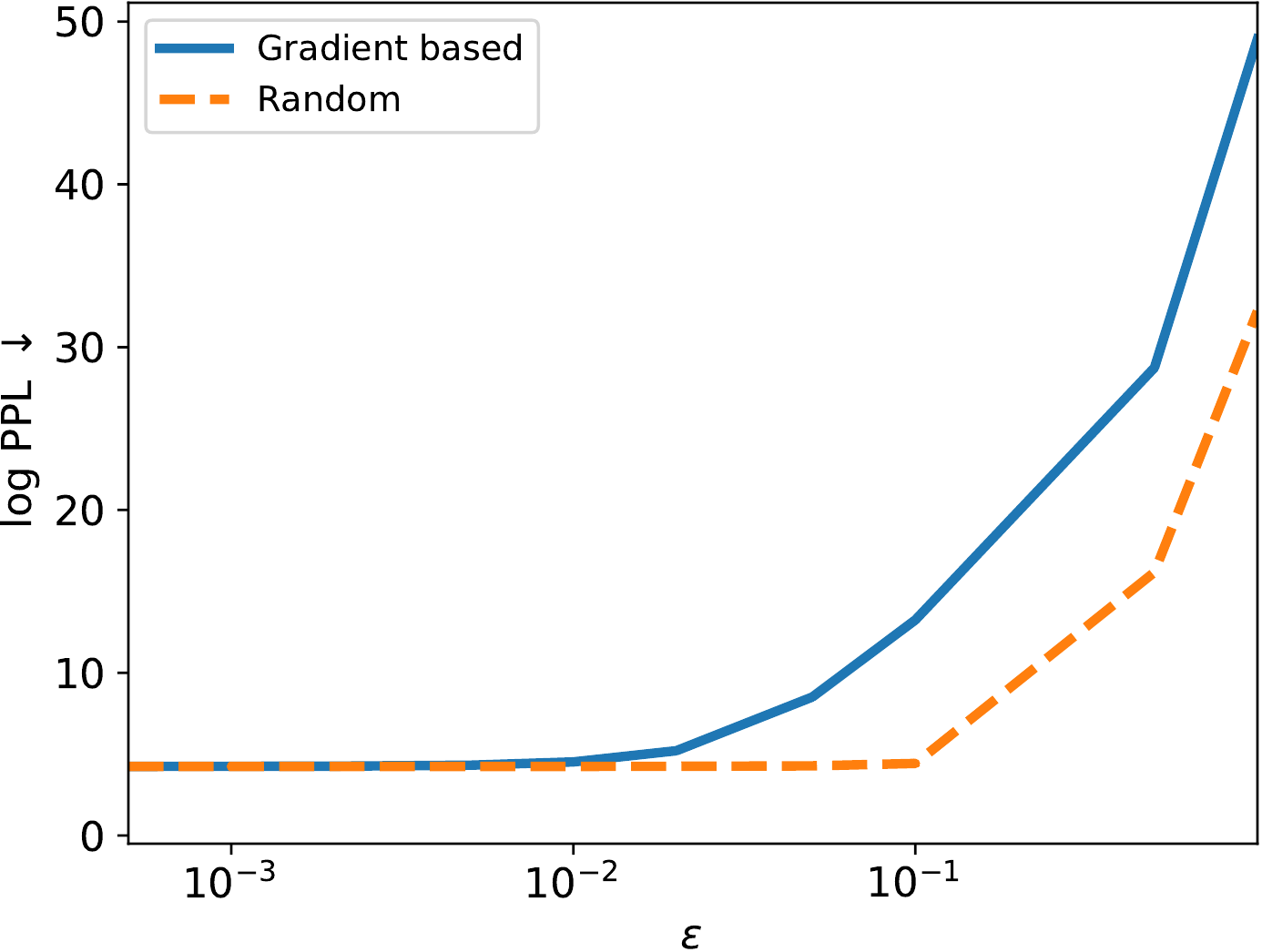}}
\hfil
\subcaptionbox{MLP ($L_{+\infty}$)}{\includegraphics[width=0.22\linewidth]{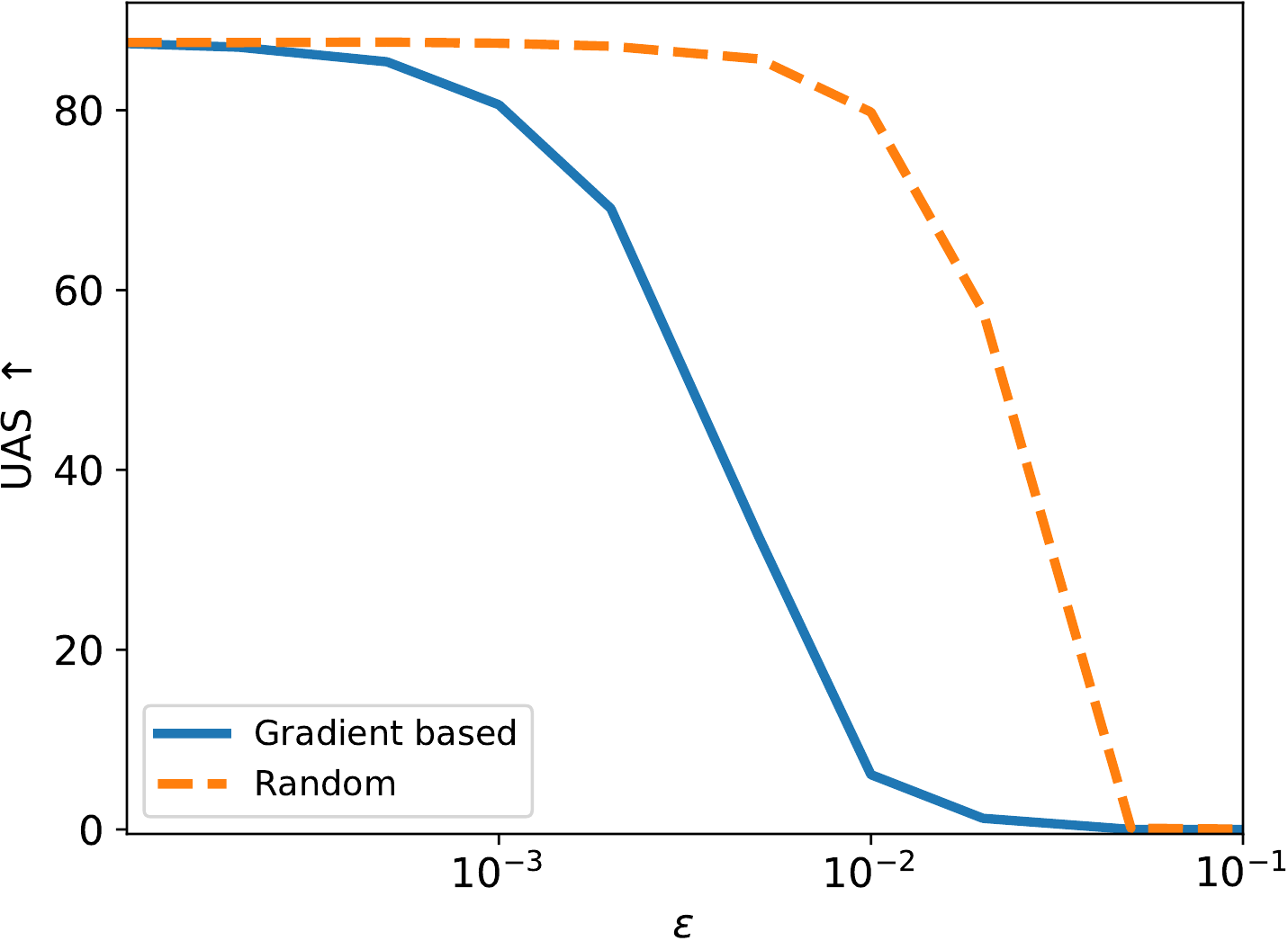}}
\caption{Results of gradient-based corruption and random corruption under the corruption constraints ($n=k$). Results of ResNet-34 are from CIFAR-10. We can conclude that the gradient-based corruption performs more effectively than the random corruption on all the tasks.}
\label{fig:compare with ramdom}
\end{figure*}


\subsection{Gradient-Based Corruption}

To arrive at the optimal parameter corruption that renders a more accurate estimation of the proposed indicator, we further propose a gradient-based method based on maximizing the first-order estimation $f(\vect{a})=\vect{a}^\text{T}\vect{g}$ of the indicator.

\begin{defn}[Gradient-Based Corruption and Estimation]
\label{def:2}

Maximizing the first-order estimation $f(\vect{a})=\vect{a}^\text{T}\vect{g}$ of the indicator, the gradient-based parameter corruption $\vect{\hat a}$ in $S$ is
\begin{align}
\vect{\hat a}=\argmax_{\vect{a}\in S}\vect{a}^\text{T}\vect{g} &= \epsilon\left(\text{sgn}(\vect{h})\odot\frac{|\vect{h}|^\frac{1}{p-1}}{\||\vect{h}|^\frac{1}{p-1}\|_p}\right); 
\label{equ:proposed}\\ f(\vect{\hat a})&=\vect{\hat a}^\text{T}\vect{g}=\epsilon\|\vect{h}\|_{\frac{p}{p-1}};
\end{align}
where $\vect{h}=top_n(\vect{g})$, other notations are used similarly to Definition~\ref{def:1}. The resultant corruption vector leads to a gradient-based estimation of the indicator.
\end{defn}

The procedure of the gradient-based method is summarized in Algorithm~\ref{alg:gradient}. The error bound of the gradient-based estimation is described in Theorem~\ref{thm:bound}. The proof and further analysis of computational complexity are in Appendix.

\begin{thm}[Error Bound of the Gradient-Based Estimation]
\label{thm:bound}
Suppose $\mathcal{L}(\vect{w};\mathcal{D})$ is convex and $L$-smooth with respect to $\vect{w}$ in the subspace $\{\vect{w}+\vect{a}:\vect{a}\in S\}$, where $S=\{\vect{a}:\|\vect{a}\|_p=\epsilon\text{ and }\|\vect{a}\|_0\le n\}$.\footnote{Note that $\mathcal{L}$ is only required to be convex and $L$-smooth in a neighbourhood of $\vect{w}$, instead of the entire $\mathbb{R}^k$.} Suppose $\vect{a^*}$ and $\vect{\hat a}$ are the optimal corruption and the gradient-based corruption in $S$ respectively. $\|\vect{g}\|_2=G>0$. It is easy to verify that $\mathcal{L}(\vect{w}+\vect{a^*};\mathcal{D})\ge \mathcal{L}(\vect{w+\vect{\hat a}};\mathcal{D})>\mathcal{L}(\vect{w};\mathcal{D})$ . It can be proved that the loss change of the gradient-based corruption is the same order infinitesimal of that of the optimal parameter corruption:
\begin{equation}
\frac{\Delta_\text{max}\mathcal{L}(\vect{w}, S; \mathcal{D})}{\Delta\mathcal{L}(\vect{w}, \vect{\hat a}; \mathcal{D})}=1+O\left(\frac{Ln^{g(p)}\sqrt{k}\epsilon}{G}\right);
\label{eq:bound}
\end{equation}
where $g(p)$ is formulated as $g(p)=\max\{\frac{p-4}{2p}, \frac{1-p}{p}\}$.
\end{thm}

Theorem~\ref{thm:bound} guarantees when perturbations to model parameters are small enough, the gradient-based corruption can accurately estimate the indicator with small errors. In Eq.(\ref{eq:bound}), the numerator is the proposed indicator, which is the maximum loss change caused by parameter corruption, and the denominator is the loss change with the parameter corruption generated by the gradient-based method. As we can see, when  $\epsilon$, the $p$-norm of the corruption vector, tends to zero, the term $O(\cdot)$ will also tend to zero such that the 
ratio becomes one, meaning the gradient-based method is an infinitesimal estimation of the indicator.

\section{Experiments}
We first empirically validate the effectiveness of the proposed gradient-based corruption compared to random corruption. Then, it is applied to evaluate the robustness of neural network parameters by scanning for vulnerability and counteract parameter corruption via adversarial training.

\subsection{Experimental Settings}

We use four widely-used tasks including benchmark datasets in CV and NLP and use diverse neural network architecture. On the image classification task,
the base model is ResNet-34 \cite{he2016deep}, the datasets are CIFAR-10 \cite{torralba200880} and ImageNet, and the evaluation metric is accuracy. On the machine translation task, the base model is Transformer provided by fairseq~\cite{ott2019fairseq}, the dataset is German-English translation dataset (De-En) \citet{ott2019fairseq,DBLP:journals/corr/RanzatoCAZ15,DBLP:conf/emnlp/WisemanR16}, and the evaluation metric is BLEU score. On the language modeling task, the base model is LSTM following \cite{merityRegOpt,merityAnalysis}, the dataset is the English Penn TreeBank (PTB-LM) \citep{DBLP:journals/coling/MarcusSM94}, and the evaluation metric is Log Perplexity (Log PPL). On the dependency parsing task, the base model is MLP following \citet{DBLP:conf/emnlp/ChenM14}, the dataset is the English Penn TreeBank dependency parsing (PTB-Parsing) \citep{DBLP:journals/coling/MarcusSM94}, and the evaluation metric is Unlabeled Attachment Score (UAS). For the detailed experimental setup, please refer to Appendix.

\subsection{Validation of Gradient-Based Corruption}
\label{sec:exp:grad-rand}

The comparative results between the gradient-based corruption and the random corruption are shown in Figure~\ref{fig:compare with ramdom} and Table~\ref{tab:rand-results}. Figure~\ref{fig:compare with ramdom} shows that parameter corruption under the corruption constraint can result in substantial accuracy degradation for different sorts of neural networks and the gradient-based parameter corruption requires smaller perturbation than the random parameter corruption. The gradient-based corruption works for smaller corruption length and causes more damage at the same corruption length. 
To conclude, the gradient-based corruption effectively defects model parameters with minimal corruption length compared to the random corruption, thus being a viable and efficient approach to find the parameter vulnerability.

\subsection{Probing the Vulnerability of DNN Parameters}
Here we use the indicator to probe the Vulnerability of DNN Parameters. We use the gradient-based corruption on parameters from separated components and set $n$ as the maximum number of the corrupted parameters. We probe the vulnerability of network parameters in terms of two natural structural characteristics of deep neural networks: the type, e.g., whether they belong to embeddings or convolutions, and the position, e.g., whether they belong to lower layers or higher layers. Due to limited space, the results of different layers in neural networks and detailed visualization of the vulnerability of different components are shown in Appendix.

\paragraph{Vulnerability in Terms of Parameter Types}
Figure~\ref{fig:components} (a-b) show the distinguished vulnerability of different selected components in ResNet-34 and Transformer. Several observations can be drawn from the results:
\textit{(1) Normalization layers are prone to parameter corruption.} The batch normalization in ResNet-34 and the layer normalization in Transformer are most sensitive in comparison to other components in each network. It is possible that since these components adjust the data distribution, a slight change in scaling or biasing could lead to systematic disorder in the whole network.
\textit{(2) Convolution layers are more sensitive to corruption than fully-connected layers.} Since parameters in convolution, i.e., the filters, are repeatedly applied to the input feature grids, they might exert more influence than parameters in fully-connected layers that are only applied to the inputs once.
\textit{(3) Embedding and attention layers are relatively robust against parameter corruption.} It is obvious that embeddings consist of word vectors and fewer word vectors are corrupted if the corrupted number of parameters is limited, thus scarcely affecting the model. The robustness of attention is intriguing and further experimentation is required to understand its characteristics.


\begin{figure*}[!t]
\centering
\begin{minipage}{.45\linewidth}
\centering
\subcaptionbox{ResNet-34 }{\includegraphics[width=0.48\linewidth]{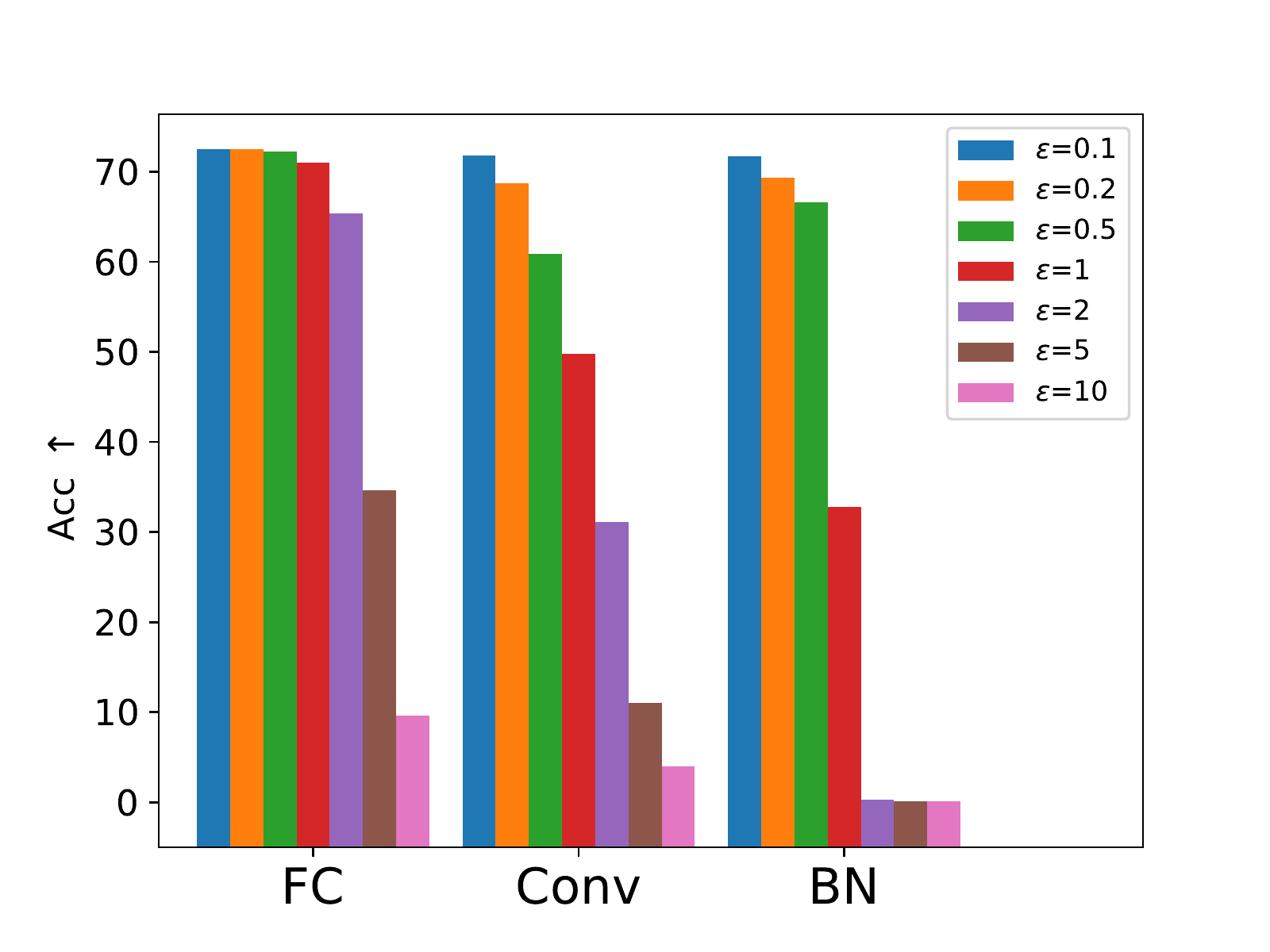}}
\hfil
\subcaptionbox{Transformer}{\includegraphics[width=0.48\linewidth]{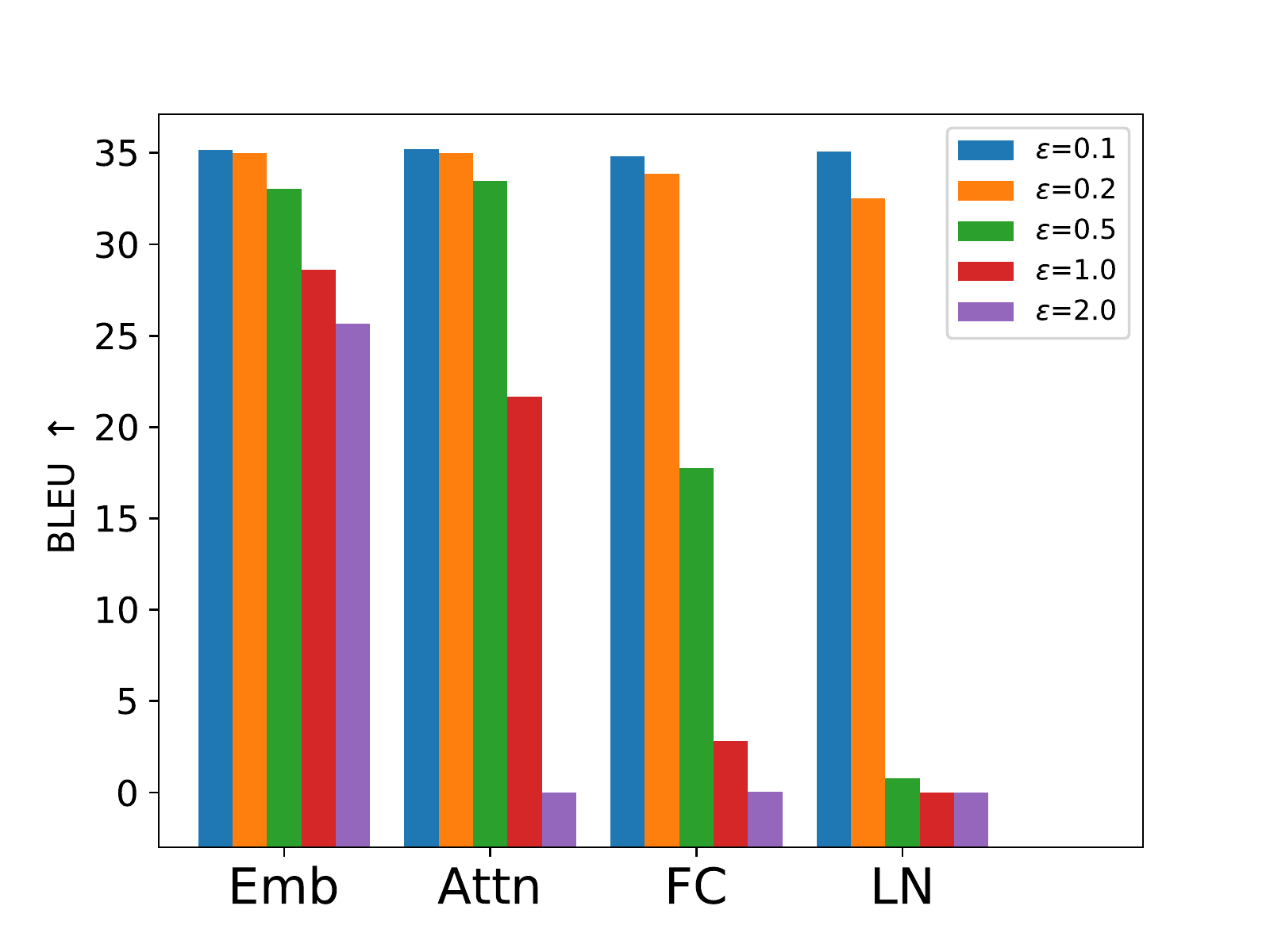}}

\subcaptionbox{ResNet-34 }{\includegraphics[width=0.48\linewidth]{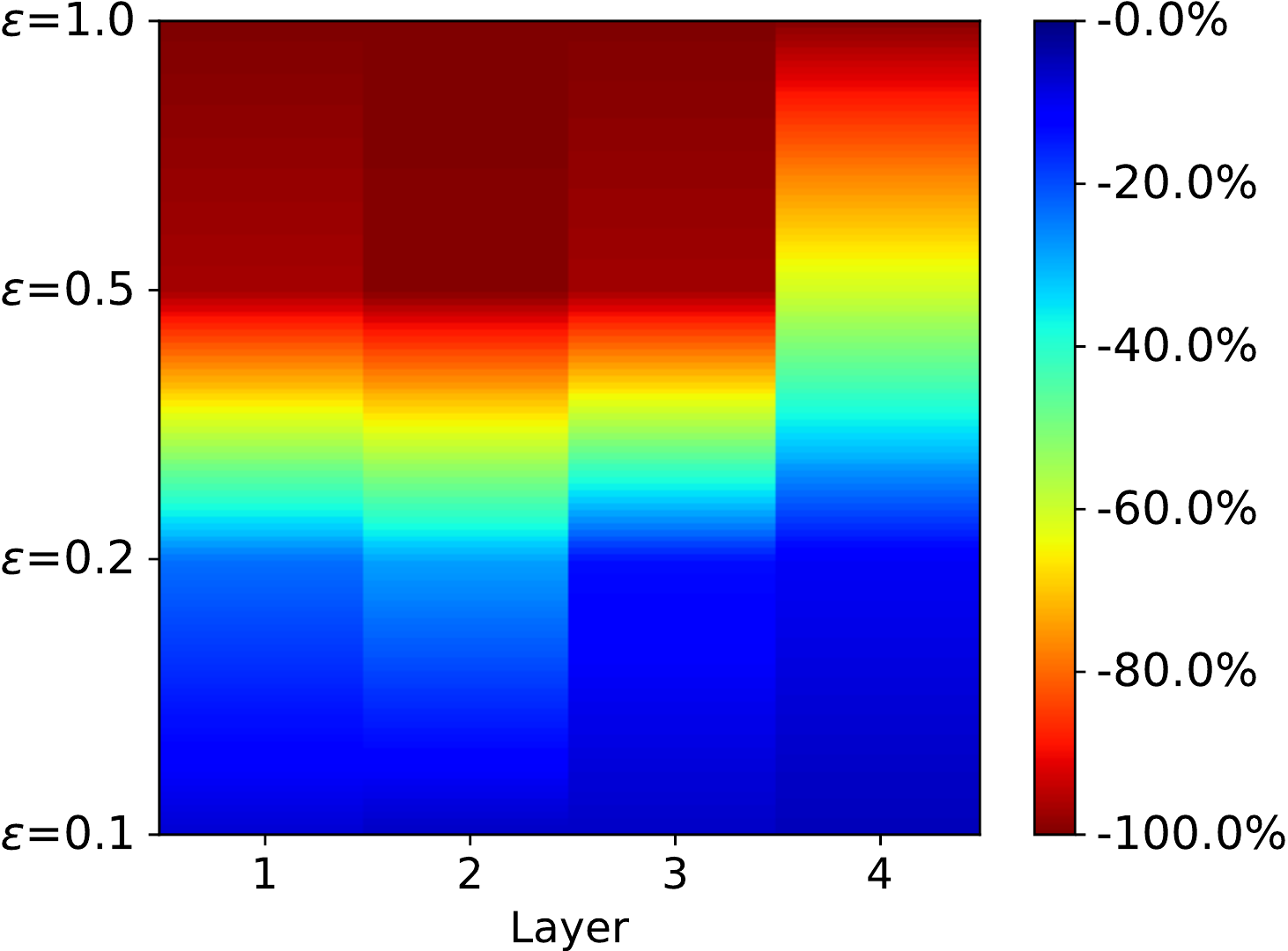}}
\hfil
\subcaptionbox{Transformer-Dec}{\includegraphics[width=0.48\linewidth]{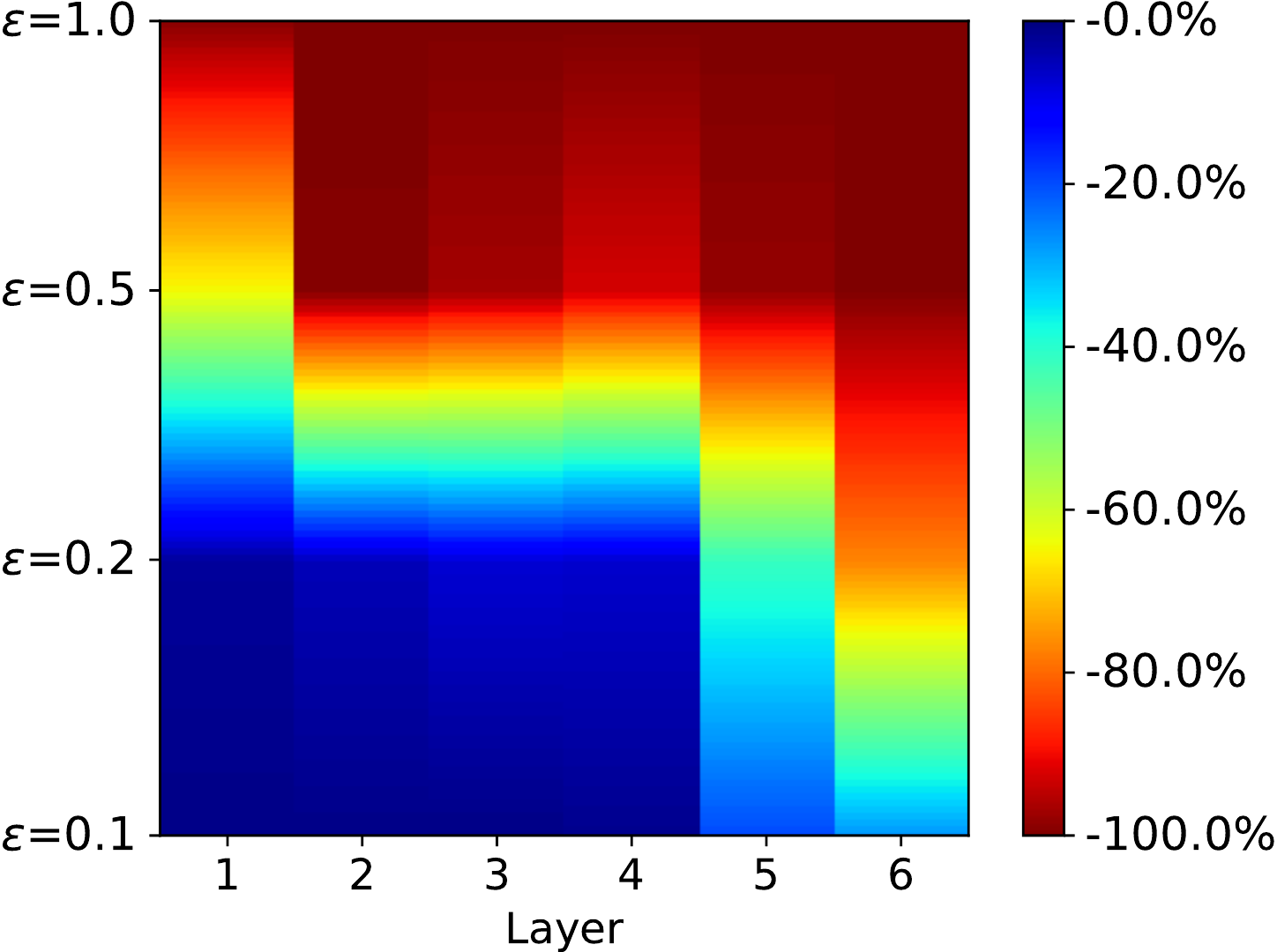}}
\end{minipage}
\hfil
\begin{minipage}{0.52\linewidth}
\subcaptionbox{ResNet-34 }[0.42\linewidth]{\includegraphics[height=16.5\baselineskip]{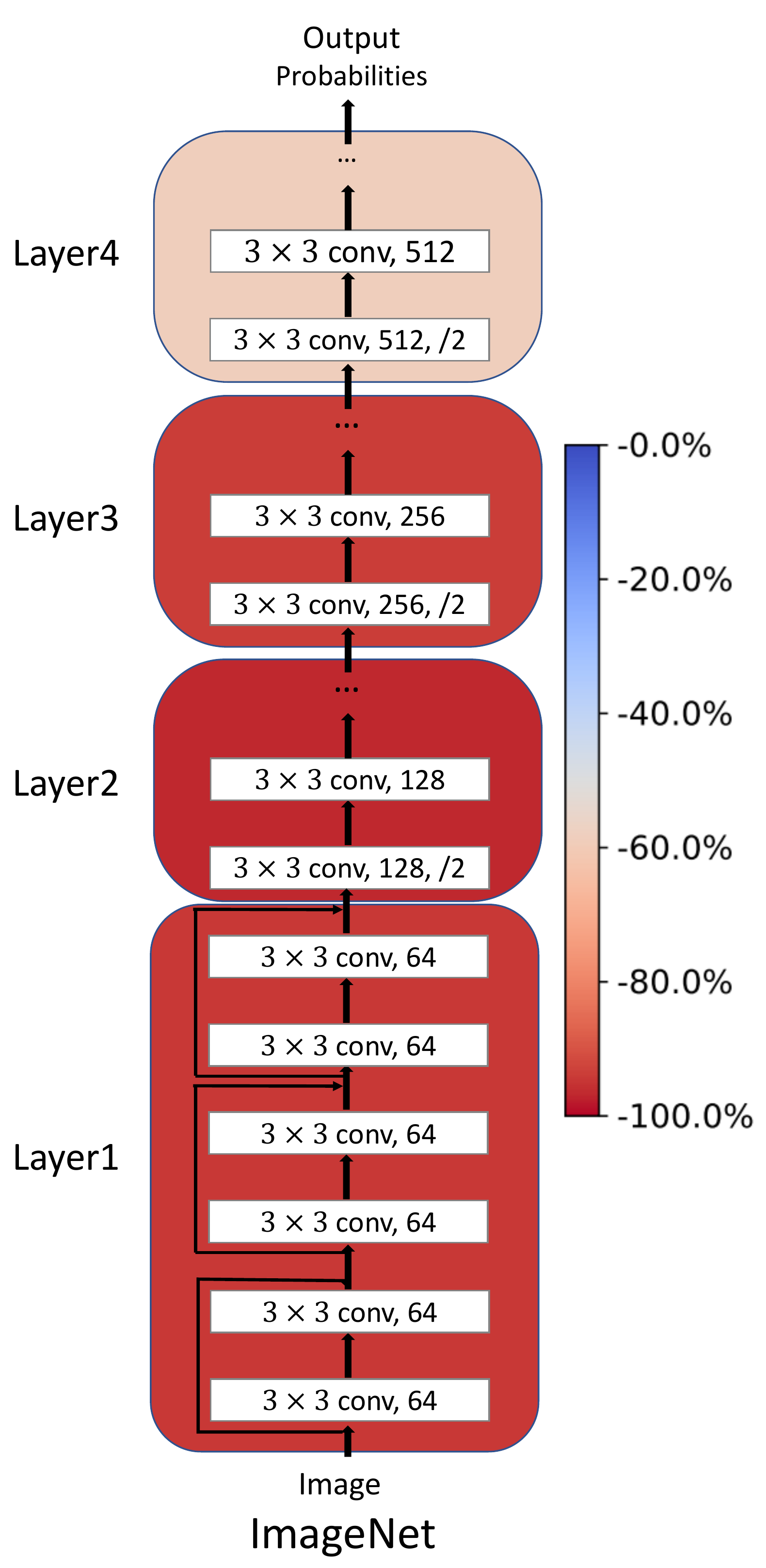}}
\hfil
\subcaptionbox{Transformer}[0.57\linewidth]{\includegraphics[height=16.5\baselineskip]{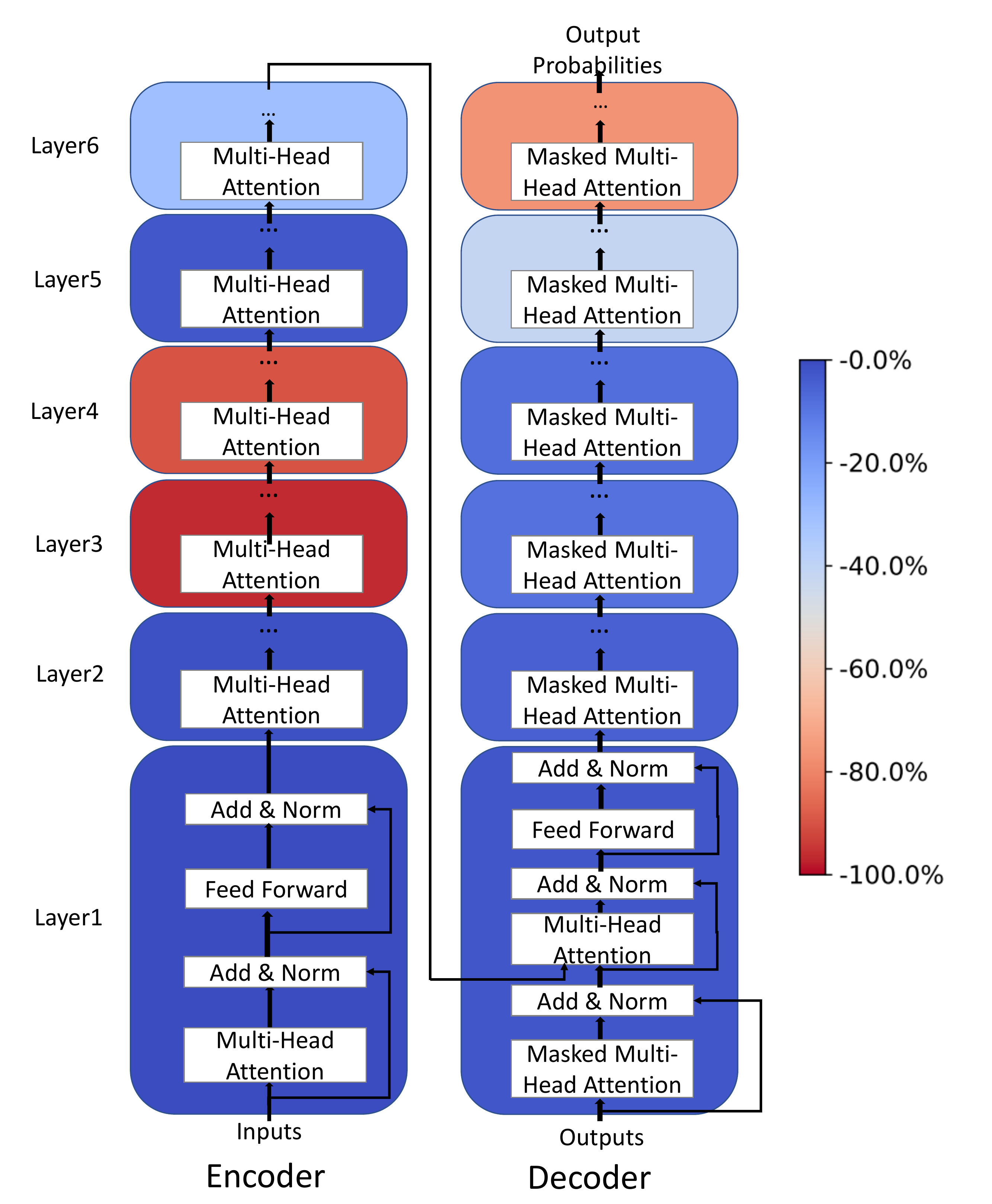}}
\end{minipage}
\caption{
Results of gradient-based corruption on (a-b) different components of ResNet-34 (ImageNet) and Transformer under the corruption constraint ($n=10$, $L_{+\infty}$-norm) and (c-f) different layers in ResNet-34 and Transformer ($n=100$, $L_{+\infty}$-norm). Conv: convolution; Emb: embedding; FC: fully-connected; Attn: attention; BN: batch normalization; LN: layer normalization. $\epsilon$ is set to be 10 and 0.2 for ResNet-34 and Transformer, respectively.  Warmer colors indicate significant accuracy degradation.
}
\label{fig:components}
\end{figure*}


\paragraph{Vulnerability in Terms of Parameter Positions}
The illustration of division of different layers and results of parameter corruption on different layers are shown in Figure~\ref{fig:components} (c-f). We can draw the following observations:
\textit{(1) Lower layers in ResNet-34 are less robust to parameter corruption.} It is generally believed that lower layers in convolutional neural networks extract basic visual patterns and are very fundamental in classification tasks \cite{yosinki2014transfer}, which indicates that perturbations to lower layers can fundamentally hurt the whole model.
\textit{(2) Upper layers in Transformer Decoder are less robust to parameter corruption.} From the sequence-to-sequence perspective, the encoder layers encode the sequence from shallow semantics to deep semantics and the decoder layers decode the sequence in a reversed order. It means that the higher layers are responsible for the choice of specific words and have a direct impact on the generated sequence. For Transformer Encoder, the parameter corruption exhibits inconspicuous trends.

As we can see, the proposed indicator reveals several problems that are rarely paid attention to before. Especially, the results on normalization layers should provide verification for the heuristic designs of future architecture.

\section{Adversarial Corruption-Resistant Training}

As shown by the probing results, the indicator can reveal interesting vulnerability of neural networks, which leads to poor robustness against parameter corruption. An important question is what we could do about the discovered vulnerability in practice, since it could be the innate characteristic of the neural network components and cannot be eliminated in design. However, if we can automatically drive the parameters from the area with steep surroundings measured by the indicator, we can obtain models that achieve natural balance on accuracy and parameter robustness.

\subsection{Adversarial Corruption-Resistant Loss}
To this end, we propose the adversarial corruption-resistant loss $\mathcal{L}_a$ to counteract parameter corruptions in an adversarial way. The key idea is to routinely corrupt the parameters and minimize both the induced loss change and the original loss. Intuitively, the proposed method tries to keep the parameters away from the neighborhood where there are sheer directions around, which means the parameters should be situated at the center of a flattish basin in the loss landscape.

Concretely, given batched data $\mathcal{B}$, \textit{virtual} gradient-based corruption $\vect{\hat a}$ on parameter $\vect{w}$, we propose to minimize both the loss with corrupted parameter $\vect{w}+\vect{\hat a}$ and the original loss by minimizing a new loss $\mathcal{L}^*(\vect{w}; \mathcal{B})$:
\begin{align}
\mathcal{L}^*(\vect{w}; \mathcal{B})&=(1-\alpha) \mathcal{L}(\vect{w}; \mathcal{B})+\alpha \mathcal{L}(\vect{w}+\vect{\hat a}; \mathcal{B})\label{eq:freeLB}
\\ &\approx (1-\alpha)\mathcal{L}(\vect{w}; \mathcal{B})+\alpha [\mathcal{L}(\vect{w}; \mathcal{B})+f(\vect{\hat a})]\\ &= \mathcal{L}(\vect{w}; \mathcal{B})+\alpha f(\vect{\hat a})\label{eq:norm}.
\end{align}

According to Eq.(\ref{equ:proposed}), when $S=\{\|\vect{a}\|_p=\epsilon\}$, $f(\vect{\hat a})$ can be written as $ f(\vect{\hat a})=\epsilon\|\vect{g}\|_{p/(p-1)}$, where $\vect{g}=\nabla_\vect{w}\mathcal{L}(\vect{w}; \mathcal{B})$, which can be seen as a regularization term in the proposed adversarial corruption-resistant training. We can see that it actually serves as \textit{gradient regularization} by simple derivation. Therefore, we define the adversarial corruption-resistant loss $\mathcal{L}_a(\vect{w}; \mathcal{B})$ as
\begin{align}
\mathcal{L}_a(\vect{w}; \mathcal{B})=\mathcal{L}(\vect{w}; \mathcal{B})+\lambda\|\vect{g}\|_{p/(p-1)}
\end{align}
where $\mathcal{L}_a$ is equivalent to Eq.(\ref{eq:norm}) when $\lambda=\alpha\epsilon$. \citet{L1} adopts the $L_1$-norm of gradients as regularization term to improve the robustness of model against quantization, which can be treated as the $L_{+\infty}$ bounded parameter corruption. In our proposed universal framework, we adopt the $L_{p/(p-1)}$-norm of gradients as regularization term to resist the $L_p$-norm bounded parameter corruption.

\begin{table*}[t]
\centering
\scriptsize
\setlength{\tabcolsep}{2 pt}
\begin{tabular}{@{}llcclcclcclcclcc@{}}
\toprule
\textbf{Dataset} & \multicolumn{3}{c}{\textbf{Tiny-ImageNet (Acc $\uparrow$)} } & \multicolumn{3}{c}{\textbf{CIFAR-10 (Acc $\uparrow$)} } & \multicolumn{3}{c}{\textbf{PTB-LM (PPL $\downarrow$)}} & \multicolumn{3}{c}{\textbf{PTB-Parsing (UAS $\uparrow$)}} & \multicolumn{3}{c}{\textbf{De-En (BLEU $\uparrow$)}}  \\
\cmidrule(r){1-1}\cmidrule(r){2-4}\cmidrule(r){5-7}\cmidrule(r){8-10}\cmidrule(r){11-13}\cmidrule(r){14-16}
Base model &\multicolumn{6}{c}{ResNet-34} & \multicolumn{3}{c}{LSTM} & \multicolumn{3}{c}{MLP} & \multicolumn{3}{c}{Transformer} \\
\midrule
Approach & $\epsilon$ & Baseline & Proposed & $\epsilon$ & Baseline & Proposed & $\epsilon$ & Baseline & Proposed & $\epsilon$ & Baseline & Proposed & $\epsilon$ & Baseline & Proposed \\
\cmidrule(r){1-1}\cmidrule(r){2-4}\cmidrule(r){5-7}\cmidrule(r){8-10}\cmidrule(r){11-13}\cmidrule(r){14-16}
w/o corruption & - & 66.56 & \textbf{67.06 (+0.50)} & - & 94.26 & \bf 96.12 (+1.86) &  - & \PH 70.09 & \textbf{68.28 (-1.81)} &  -  & 87.26 & \bf 87.89 (+0.63) & -  &  35.33 & \textbf{35.69 (+0.36)}\\
\cmidrule(r){1-1}\cmidrule(r){2-4}\cmidrule(r){5-7}\cmidrule(r){8-10}\cmidrule(r){11-13}\cmidrule(r){14-16}
\multirow{5}*{Corruption} & 0.05 & 64.88 & 65.79 & 0.05 & 93.20  & 95.65 & 10 & 139.23 & 69.12 & 0.005 & 86.74  & 87.67 & 0.5 & 31.68 &  34.84 \\
& 0.1 & 60.70 & 63.47 & 0.1 & 89.84  & 94.59 & 12 & 180.91 & 69.30 & 0.01 & 85.32 & 87.26  & 0.75 & 22.13 & 33.58 \\
& 0.2 & 41.79 & 52.51 & 0.2 & 71.44 & 87.92 & 14 & 240.99 & 69.61 &  0.05 & 72.32 & 82.06 & 1 & 11.17 & 31.25 \\
& 0.5 & \PH 1.07 & \PH 7.92 & 0.5 & 13.77 & 21.42 & 16 & 327.49 & 70.04 &0.1 & 61.80 & 73.19   & 1.25  & \PH 3.47 & 24.47 \\
& 1 & \PH 0.55 & \PH 1.26 & 1 & 10.00 & 10.94 & 18 & 452.03 & 70.61 & 0.2 & 49.46 & 57.73 & 1.5 & \PH 1.56 & 14.91\\
\bottomrule
\end{tabular}
\caption{
Results of the proposed corruption-resistant training, which not only improves the accuracy without corruption but also enhances the robustness against corruption. In parameter corruptions ($n=k$, $L_2$-norm), all parameters can be corrupted.}
\label{tab:eps_acc}
\end{table*}

\subsection{Relations to Resistance against Data Perturbations}

In the common $L_2$ or $L_{+\infty}$ cases, our \textit{gradient regularization} term can be written as $\|\vect{g}\|_{p/(p-1)}=\|\vect{g}\|_2$ when $p=2$, and $\|\vect{g}\|_{p/(p-1)}=\|\vect{g}\|_1$ when $p=+\infty$. 

The formulation of the gradient regularization $\mathcal{L}(\vect{w}; \mathcal{B})+\|\vect{g}\|_1$ (or $\|\vect{g}\|_2$) is similar to the weight regularization $\mathcal{L}(\vect{w}; \mathcal{B})+\|\vect{w}\|_1$ (or $\|\vect{w}\|_2$). \citet{adv-reg} indicates that $L_1$ or $L_2$ \textit{weight regularization} is equivalent to resist $L_{+\infty}$ and $L_2$ \textit{data perturbations} respectively under some circumstances. Complementarily, we show that  $L_1$ and $L_2$ \textit{gradient regularization} is equivalent to resist $L_{+\infty}$ and $L_2$ \textit{parameter corruptions}, respectively.

\subsection{Experiments} 
We conduct experiments on the above benchmark datasets to validate that the proposed corruption-resistant training functions as designed. For ImageNet, due to its huge size, we test our corruption resistant training method on a subset of ImageNet, the Tiny-ImageNet dataset. We find that optimizing $\mathcal{L}^*$ in Eq.(\ref{eq:freeLB}) directly instead of adopting the gradient regularization term can further improve the accuracy on some tasks. Therefore, we sometimes adopt a variant of $\mathcal{L}_a$ by directly optimizing $\mathcal{L}^*$ in Eq.(\ref{eq:freeLB}). Detailed experimental settings and supplemental results are reported in Appendix.

In Table~\ref{tab:eps_acc}, we can see that incorporating virtual gradient-based corruptions into adversarial training can help improve both the test accuracy and the robustness of neural networks against parameter corruption. In particular, we can see that parameters that are resistant to corruption, may entail better generalization, reflected as higher accuracy on the test set.

We also find that the accuracy of the uncorrupted neural network can often be improved substantially with small magnitude of virtual parameter corruptions. However, when the magnitude of virtual parameter corruptions grows too large, virtual parameter corruptions will harm the learning process and the accuracy the uncorrupted neural network will drop. In particular, the accuracy can be treated as a unimodal function of the magnitude of virtual parameter corruptions approximately, whose best configuration can be determined easily.

\section{Related Work}

\paragraph{Vulnerability of Deep Neural Networks}
Existing studies concerning vulnerability or robustness of neural networks mostly focus on generating adversarial examples~\citep{DBLP:journals/corr/GoodfellowSS14} and adversarial training algorithms given adversarial examples in the input data~\citep{DBLP:journals/corr/freeLB}. \citet{DBLP:journals/corr/SzegedyZSBEGF13} first proposed the concept of adversarial examples and found that neural network classifiers are vulnerable to adversarial attacks on input data. Following that study, different adversarial attack algorithms~\citep{DBLP:conf/cvpr/Moosavi-Dezfooli16,DBLP:conf/iclr/KurakinGB17a} were developed. Another class of studies~\citep{backdoor1,backdoor2,backdoors3,backdoor-Bert} known as backdoor attacks injected vulnerabilities to neural networks by data poisoning, which requires access to the training process of the neural network models.

\paragraph{Adversarial Training}
Other related work on adversarial examples aimed to design adversarial defense algorithms to evaluate and improve the robustness of neural networks over adversarial examples~\citep{DBLP:conf/sp/Carlini017,DBLP:conf/iclr/MadryMSTV18,DBLP:journals/corr/freeLB}.
As another application of adversarial training, GAN~\citep{DBLP:journals/corr/GAN} has been widely used in multiple machine learning tasks, such as computer vision~\citep{GAN8, GAN9}, natural language processing~\citep{GAN16, GAN17} and time series synthesis~\citep{GAN19, GAN21}.

\paragraph{Changes in Neural Network Parameters}
Existing studies also concern the influence of noises or changes in neural network parameters by training data poisoning~\citep{backdoor1,backdoor2,backdoors3}, bit flipping~\citep{TBT}, compression~\citep{Stronger} or quantization~\citep{Data-Free-Quant,model_robustness}. \citet{LCA} proposes the Loss Change Allocation indicator (LCA) to analyze the allocation of loss change partitioned to different parameters.

To summarize, existing related work mostly focuses on adversarial examples and its adversarial training. However, we focus on parameter corruptions of neural networks so as to find vulnerable components of models and design an adversarial corruption-resistant training algorithm to improve the parameter robustness.

\section{Conclusion}

To better understand the vulnerability of deep neural network parameters, which is not well studied before, we propose an indicator measuring the maximum loss change when a small perturbation is applied to model parameters to evaluate the robustness against parameter corruption. Intuitively, the indicator describes the steepness of the loss surface around the parameters. We show that the indicator can be efficiently estimated by a gradient-based method and random parameter corruptions can hardly induce the maximum degradation, which is validated both theoretically and empirically. In addition, we apply the proposed indicator to systematically analyze the vulnerability of different parameters in different neural networks and reveal that the normalization layers, which are important in stabilizing the data distribution in deep neural networks, are prone to parameter corruption. Furthermore, we propose an adversarial learning approach to improve the parameter robustness and show that parameters that are resistant to parameter corruption embody better robustness and accuracy.

\section{Acknowledgements}
\cc{We thank the anonymous reviewers for their constructive comments. This work is partly supported by Beijing Academy of Artificial Intelligence (BAAI).}

\section{Ethics Statement}

This paper presents a study on parameter corruptions in deep neural networks. Despite the promising performance in benchmark datasets, the existing deep neural network models are not robust in real-life scenarios and run the risks of adversarial examples, backdoor attacks~\citep{backdoor1,backdoor2,backdoors3,backdoor-Bert}, and the parameter corruption issues.

Unlike adversarial examples and backdoor attacks, parameter corruptions have drawn limited attention in the community despite its urgent need in areas such as hardware neural networks and software neural networks applied in a difficult hardware environment. Our work takes a first step towards the parameter corruptions and we are able to investigate the robustness of different model parameters and reveal vulnerability of deep neural networks. It provides fundamental guidance for applying deep neural networks in the aforementioned scenarios. Moreover, we also propose an adversarial corruption-resistant training to improve the robustness of neural networks, making such models available to many more critical applications.

On the other hand, the method used in this work to estimate the loss change could also be taken maliciously to tamper with the neural network applied in business. However, such kind of ``attack'' requires access to the storage of parameters, meaning that the system security would have been already breached. Still, it should be recommended that certain measures are taken to verify the parameters are not changed or check the parameters are corrupted in actual applications.

\small
\bibliography{aaai21}

\begin{thebibliography}{45}
\providecommand{\natexlab}[1]{#1}
\providecommand{\url}[1]{\texttt{#1}}
\providecommand{\urlprefix}{URL }
\expandafter\ifx\csname urlstyle\endcsname\relax
  \providecommand{\doi}[1]{doi:\discretionary{}{}{}#1}\else
  \providecommand{\doi}{doi:\discretionary{}{}{}\begingroup
  \urlstyle{rm}\Url}\fi

\bibitem[{Abdelsalam et~al.(2018)Abdelsalam, Boulet, Demers, Langlois, and
  Cheriet}]{abdelsalam2018efficient}
Abdelsalam, A.~M.; Boulet, F.; Demers, G.; Langlois, J. M.~P.; and Cheriet, F.
  2018.
\newblock An Efficient FPGA-based Overlay Inference Architecture for Fully
  Connected DNNs.
\newblock In \emph{2018 International Conference on ReConFigurable Computing
  and FPGAs (ReConFig)}, 1--6.

\bibitem[{Alizadeh et~al.(2020)Alizadeh, Behboodi, van Baalen, Louizos,
  Blankevoort, and Welling}]{L1}
Alizadeh, M.; Behboodi, A.; van Baalen, M.; Louizos, C.; Blankevoort, T.; and
  Welling, M. 2020.
\newblock Gradient {\textdollar}{\textbackslash}ell{\_}1{\textdollar}
  Regularization for Quantization Robustness.
\newblock In \emph{8th International Conference on Learning Representations,
  {ICLR} 2020, Addis Ababa, Ethiopia, April 26-30, 2020}. OpenReview.net.
\newblock \urlprefix\url{https://openreview.net/forum?id=ryxK0JBtPr}.

\bibitem[{Arora et~al.(2018)Arora, Ge, Neyshabur, and Zhang}]{Stronger}
Arora, S.; Ge, R.; Neyshabur, B.; and Zhang, Y. 2018.
\newblock Stronger Generalization Bounds for Deep Nets via a Compression
  Approach.
\newblock In \emph{Proceedings of the 35th International Conference on Machine
  Learning, {ICML} 2018, Stockholmsm{\"{a}}ssan, Stockholm, Sweden, July 10-15,
  2018}, 254--263.

\bibitem[{Bui and Phillips(2019)}]{bui2019scalable}
Bui, T. T.~T.; and Phillips, B. 2019.
\newblock A Scalable Network-on-Chip Based Neural Network Implementation on
  FPGAs.
\newblock In \emph{2019 IEEE-RIVF International Conference on Computing and
  Communication Technologies (RIVF)}, 1--6.

\bibitem[{Carlini and Wagner(2017)}]{DBLP:conf/sp/Carlini017}
Carlini, N.; and Wagner, D.~A. 2017.
\newblock Towards Evaluating the Robustness of Neural Networks.
\newblock In \emph{2017 {IEEE} Symposium on Security and Privacy, {SP} 2017,
  San Jose, CA, USA, May 22-26, 2017}, 39--57.
\newblock \doi{10.1109/SP.2017.49}.

\bibitem[{Chaudhari et~al.(2017)Chaudhari, Choromanska, Soatto, LeCun,
  Baldassi, Borgs, Chayes, Sagun, and Zecchina}]{sharp-local2}
Chaudhari, P.; Choromanska, A.; Soatto, S.; LeCun, Y.; Baldassi, C.; Borgs, C.;
  Chayes, J.~T.; Sagun, L.; and Zecchina, R. 2017.
\newblock Entropy-SGD: Biasing Gradient Descent Into Wide Valleys.
\newblock In \emph{5th International Conference on Learning Representations,
  {ICLR} 2017, Toulon, France, April 24-26, 2017, Conference Track
  Proceedings}. OpenReview.net.
\newblock \urlprefix\url{https://openreview.net/forum?id=B1YfAfcgl}.

\bibitem[{Chen and Manning(2014)}]{DBLP:conf/emnlp/ChenM14}
Chen, D.; and Manning, C.~D. 2014.
\newblock A Fast and Accurate Dependency Parser using Neural Networks.
\newblock In \emph{Proceedings of the 2014 Conference on Empirical Methods in
  Natural Language Processing, {EMNLP} 2014, October 25-29, 2014, Doha, Qatar,
  {A} meeting of SIGDAT, a Special Interest Group of the {ACL}}, 740--750.

\bibitem[{Chen et~al.(2017)Chen, Liu, Li, Lu, and Song}]{backdoor2}
Chen, X.; Liu, C.; Li, B.; Lu, K.; and Song, D. 2017.
\newblock Targeted Backdoor Attacks on Deep Learning Systems Using Data
  Poisoning.
\newblock \emph{CoRR} abs/1712.05526.
\newblock \urlprefix\url{http://arxiv.org/abs/1712.05526}.

\bibitem[{{Dai}, {Chen}, and {Li}(2019)}]{backdoor1}
{Dai}, J.; {Chen}, C.; and {Li}, Y. 2019.
\newblock A Backdoor Attack Against LSTM-Based Text Classification Systems.
\newblock \emph{IEEE Access} 7: 138872--138878.

\bibitem[{Dai et~al.(2017)Dai, Yang, Yang, Cohen, and Salakhutdinov}]{GAN17}
Dai, Z.; Yang, Z.; Yang, F.; Cohen, W.~W.; and Salakhutdinov, R. 2017.
\newblock Good Semi-supervised Learning That Requires a Bad {GAN}.
\newblock In \emph{Advances in Neural Information Processing Systems 30: Annual
  Conference on Neural Information Processing Systems 2017, 4-9 December 2017,
  Long Beach, CA, {USA}}, 6510--6520.

\bibitem[{Donahue, McAuley, and Puckette(2018)}]{GAN19}
Donahue, C.; McAuley, J.~J.; and Puckette, M.~S. 2018.
\newblock Synthesizing Audio with Generative Adversarial Networks.
\newblock \emph{CoRR} abs/1802.04208.

\bibitem[{Esteban, Hyland, and R{\"{a}}tsch(2017)}]{GAN21}
Esteban, C.; Hyland, S.~L.; and R{\"{a}}tsch, G. 2017.
\newblock Real-valued (Medical) Time Series Generation with Recurrent
  Conditional GANs.
\newblock \emph{CoRR} abs/1706.02633.

\bibitem[{Feldmann et~al.(2019)Feldmann, Youngblood, Wright, Bhaskaran, and
  Pernice}]{optical-NN}
Feldmann, J.; Youngblood, N.; Wright, C.~D.; Bhaskaran, H.; and Pernice, W.
  2019.
\newblock All-optical spiking neurosynaptic networks with self-learning
  capabilities.
\newblock \emph{Nature} 569(7755): 208--214.

\bibitem[{Goodfellow et~al.(2014)Goodfellow, Pouget{-}Abadie, Mirza, Xu,
  Warde{-}Farley, Ozair, Courville, and Bengio}]{DBLP:journals/corr/GAN}
Goodfellow, I.~J.; Pouget{-}Abadie, J.; Mirza, M.; Xu, B.; Warde{-}Farley, D.;
  Ozair, S.; Courville, A.~C.; and Bengio, Y. 2014.
\newblock Generative Adversarial Networks.
\newblock \emph{CoRR} abs/1406.2661.

\bibitem[{Goodfellow, Shlens, and
  Szegedy(2015)}]{DBLP:journals/corr/GoodfellowSS14}
Goodfellow, I.~J.; Shlens, J.; and Szegedy, C. 2015.
\newblock Explaining and Harnessing Adversarial Examples.
\newblock In \emph{3rd International Conference on Learning Representations,
  {ICLR} 2015, San Diego, CA, USA, May 7-9, 2015, Conference Track
  Proceedings}.

\bibitem[{Gu et~al.(2019)Gu, Liu, Dolan{-}Gavitt, and Garg}]{backdoors3}
Gu, T.; Liu, K.; Dolan{-}Gavitt, B.; and Garg, S. 2019.
\newblock BadNets: Evaluating Backdooring Attacks on Deep Neural Networks.
\newblock \emph{IEEE Access} 7: 47230--47244.
\newblock \doi{10.1109/ACCESS.2019.2909068}.

\bibitem[{He et~al.(2016)He, Zhang, Ren, and Sun}]{he2016deep}
He, K.; Zhang, X.; Ren, S.; and Sun, J. 2016.
\newblock Deep residual learning for image recognition.
\newblock In \emph{Proceedings of the IEEE conference on computer vision and
  pattern recognition}, 770--778.

\bibitem[{Keskar et~al.(2017)Keskar, Mudigere, Nocedal, Smelyanskiy, and
  Tang}]{sharp-local1}
Keskar, N.~S.; Mudigere, D.; Nocedal, J.; Smelyanskiy, M.; and Tang, P. T.~P.
  2017.
\newblock On Large-Batch Training for Deep Learning: Generalization Gap and
  Sharp Minima.
\newblock In \emph{5th International Conference on Learning Representations,
  {ICLR} 2017, Toulon, France, April 24-26, 2017, Conference Track
  Proceedings}. OpenReview.net.
\newblock \urlprefix\url{https://openreview.net/forum?id=H1oyRlYgg}.

\bibitem[{Kurakin, Goodfellow, and Bengio(2017)}]{DBLP:conf/iclr/KurakinGB17a}
Kurakin, A.; Goodfellow, I.~J.; and Bengio, S. 2017.
\newblock Adversarial examples in the physical world.
\newblock In \emph{5th International Conference on Learning Representations,
  {ICLR} 2017, Toulon, France, April 24-26, 2017, Workshop Track Proceedings}.

\bibitem[{Kurita, Michel, and Neubig(2020)}]{backdoor-Bert}
Kurita, K.; Michel, P.; and Neubig, G. 2020.
\newblock Weight Poisoning Attacks on Pre-trained Models.
\newblock \emph{CoRR} abs/2004.06660.
\newblock \urlprefix\url{https://arxiv.org/abs/2004.06660}.

\bibitem[{Lan et~al.(2019)Lan, Liu, Zhou, and Yosinski}]{LCA}
Lan, J.; Liu, R.; Zhou, H.; and Yosinski, J. 2019.
\newblock {LCA:} Loss Change Allocation for Neural Network Training.
\newblock In \emph{Advances in Neural Information Processing Systems 32: Annual
  Conference on Neural Information Processing Systems 2019, NeurIPS 2019, 8-14
  December 2019, Vancouver, BC, Canada}, 3614--3624.

\bibitem[{Ma et~al.(2017)Ma, Jia, Sun, Schiele, Tuytelaars, and Gool}]{GAN8}
Ma, L.; Jia, X.; Sun, Q.; Schiele, B.; Tuytelaars, T.; and Gool, L.~V. 2017.
\newblock Pose Guided Person Image Generation.
\newblock In \emph{Advances in Neural Information Processing Systems 30: Annual
  Conference on Neural Information Processing Systems 2017, 4-9 December 2017,
  Long Beach, CA, {USA}}, 406--416.

\bibitem[{Madry et~al.(2018)Madry, Makelov, Schmidt, Tsipras, and
  Vladu}]{DBLP:conf/iclr/MadryMSTV18}
Madry, A.; Makelov, A.; Schmidt, L.; Tsipras, D.; and Vladu, A. 2018.
\newblock Towards Deep Learning Models Resistant to Adversarial Attacks.
\newblock In \emph{6th International Conference on Learning Representations,
  {ICLR} 2018, Vancouver, BC, Canada, April 30 - May 3, 2018, Conference Track
  Proceedings}.

\bibitem[{Marcus, Santorini, and
  Marcinkiewicz(1993)}]{DBLP:journals/coling/MarcusSM94}
Marcus, M.~P.; Santorini, B.; and Marcinkiewicz, M.~A. 1993.
\newblock Building a Large Annotated Corpus of English: The Penn Treebank.
\newblock \emph{Computational Linguistics} 19(2): 313--330.

\bibitem[{Merity, Keskar, and Socher(2017)}]{merityRegOpt}
Merity, S.; Keskar, N.~S.; and Socher, R. 2017.
\newblock {Regularizing and Optimizing LSTM Language Models}.
\newblock \emph{arXiv preprint arXiv:1708.02182} .

\bibitem[{Merity, Keskar, and Socher(2018)}]{merityAnalysis}
Merity, S.; Keskar, N.~S.; and Socher, R. 2018.
\newblock {An Analysis of Neural Language Modeling at Multiple Scales}.
\newblock \emph{arXiv preprint arXiv:1803.08240} .

\bibitem[{Misra and Saha(2010)}]{misra2010artificial}
Misra, J.; and Saha, I. 2010.
\newblock Artificial neural networks in hardware: {A} survey of two decades of
  progress.
\newblock \emph{Neurocomputing} 74(1-3): 239--255.
\newblock \doi{10.1016/j.neucom.2010.03.021}.
\newblock \urlprefix\url{https://doi.org/10.1016/j.neucom.2010.03.021}.

\bibitem[{Moosavi{-}Dezfooli, Fawzi, and
  Frossard(2016)}]{DBLP:conf/cvpr/Moosavi-Dezfooli16}
Moosavi{-}Dezfooli, S.; Fawzi, A.; and Frossard, P. 2016.
\newblock DeepFool: {A} Simple and Accurate Method to Fool Deep Neural
  Networks.
\newblock In \emph{2016 {IEEE} Conference on Computer Vision and Pattern
  Recognition, {CVPR} 2016, Las Vegas, NV, USA, June 27-30, 2016}, 2574--2582.
\newblock \doi{10.1109/CVPR.2016.282}.

\bibitem[{Nagel et~al.(2019)Nagel, van Baalen, Blankevoort, and
  Welling}]{Data-Free-Quant}
Nagel, M.; van Baalen, M.; Blankevoort, T.; and Welling, M. 2019.
\newblock Data-Free Quantization Through Weight Equalization and Bias
  Correction.
\newblock In \emph{2019 {IEEE/CVF} International Conference on Computer Vision,
  {ICCV} 2019, Seoul, Korea (South), October 27 - November 2, 2019},
  1325--1334. {IEEE}.
\newblock \doi{10.1109/ICCV.2019.00141}.
\newblock \urlprefix\url{https://doi.org/10.1109/ICCV.2019.00141}.

\bibitem[{Ott et~al.(2019)Ott, Edunov, Baevski, Fan, Gross, Ng, Grangier, and
  Auli}]{ott2019fairseq}
Ott, M.; Edunov, S.; Baevski, A.; Fan, A.; Gross, S.; Ng, N.; Grangier, D.; and
  Auli, M. 2019.
\newblock fairseq: A fast, extensible toolkit for sequence modeling.
\newblock \emph{arXiv preprint arXiv:1904.01038} .

\bibitem[{Pennington, Socher, and Manning(2014)}]{glove}
Pennington, J.; Socher, R.; and Manning, C.~D. 2014.
\newblock Glove: Global Vectors for Word Representation.
\newblock In \emph{Proceedings of the 2014 Conference on Empirical Methods in
  Natural Language Processing, {EMNLP} 2014, October 25-29, 2014, Doha, Qatar,
  {A} meeting of SIGDAT, a Special Interest Group of the {ACL}}, 1532--1543.

\bibitem[{Rakin, He, and Fan(2020)}]{TBT}
Rakin, A.~S.; He, Z.; and Fan, D. 2020.
\newblock {TBT:} Targeted Neural Network Attack With Bit Trojan.
\newblock In \emph{2020 {IEEE/CVF} Conference on Computer Vision and Pattern
  Recognition, {CVPR} 2020, Seattle, WA, USA, June 13-19, 2020}, 13195--13204.
  {IEEE}.
\newblock \doi{10.1109/CVPR42600.2020.01321}.
\newblock \urlprefix\url{https://doi.org/10.1109/CVPR42600.2020.01321}.

\bibitem[{Ranzato et~al.(2016)Ranzato, Chopra, Auli, and
  Zaremba}]{DBLP:journals/corr/RanzatoCAZ15}
Ranzato, M.; Chopra, S.; Auli, M.; and Zaremba, W. 2016.
\newblock Sequence Level Training with Recurrent Neural Networks.
\newblock In \emph{4th International Conference on Learning Representations,
  {ICLR} 2016, San Juan, Puerto Rico, May 2-4, 2016, Conference Track
  Proceedings}.

\bibitem[{Russakovsky et~al.(2015)Russakovsky, Deng, Su, Krause, Satheesh, Ma,
  Huang, Karpathy, Khosla, Bernstein, Berg, and Li}]{imagenet}
Russakovsky, O.; Deng, J.; Su, H.; Krause, J.; Satheesh, S.; Ma, S.; Huang, Z.;
  Karpathy, A.; Khosla, A.; Bernstein, M.~S.; Berg, A.~C.; and Li, F. 2015.
\newblock ImageNet Large Scale Visual Recognition Challenge.
\newblock \emph{Int. J. Comput. Vis.} 115(3): 211--252.
\newblock \doi{10.1007/s11263-015-0816-y}.
\newblock \urlprefix\url{https://doi.org/10.1007/s11263-015-0816-y}.

\bibitem[{Salimi-Nezhad et~al.(2019)Salimi-Nezhad, Ilbeigi, Amiri, Falotico,
  and Laschi}]{salimi2019digital}
Salimi-Nezhad, N.; Ilbeigi, E.; Amiri, M.; Falotico, E.; and Laschi, C. 2019.
\newblock A Digital Hardware System for Spiking Network of Tactile Afferents.
\newblock \emph{Frontiers in Neuroscience} 13.

\bibitem[{Shaham, Yamada, and Negahban(2015)}]{adv-reg}
Shaham, U.; Yamada, Y.; and Negahban, S. 2015.
\newblock Understanding Adversarial Training: Increasing Local Stability of
  Neural Nets through Robust Optimization.
\newblock \emph{CoRR} abs/1511.05432.
\newblock \urlprefix\url{http://arxiv.org/abs/1511.05432}.

\bibitem[{Szegedy et~al.(2014)Szegedy, Zaremba, Sutskever, Bruna, Erhan,
  Goodfellow, and Fergus}]{DBLP:journals/corr/SzegedyZSBEGF13}
Szegedy, C.; Zaremba, W.; Sutskever, I.; Bruna, J.; Erhan, D.; Goodfellow,
  I.~J.; and Fergus, R. 2014.
\newblock Intriguing properties of neural networks.
\newblock In \emph{2nd International Conference on Learning Representations,
  {ICLR} 2014, Banff, AB, Canada, April 14-16, 2014, Conference Track
  Proceedings}.

\bibitem[{Torralba, Fergus, and Freeman(2008)}]{torralba200880}
Torralba, A.; Fergus, R.; and Freeman, W.~T. 2008.
\newblock 80 million tiny images: A large data set for nonparametric object and
  scene recognition.
\newblock \emph{IEEE transactions on pattern analysis and machine intelligence}
  30(11): 1958--1970.

\bibitem[{Vondrick, Pirsiavash, and Torralba(2016)}]{GAN9}
Vondrick, C.; Pirsiavash, H.; and Torralba, A. 2016.
\newblock Generating Videos with Scene Dynamics.
\newblock In \emph{Advances in Neural Information Processing Systems 29: Annual
  Conference on Neural Information Processing Systems 2016, December 5-10,
  2016, Barcelona, Spain}, 613--621.

\bibitem[{Weber, da~Silva~Labres, and Cabrera(2019)}]{weber2019amplifier}
Weber, T.~O.; da~Silva~Labres, D.; and Cabrera, F.~L. 2019.
\newblock Amplifier-based MOS Analog Neural Network Implementation and Weights
  Optimization.
\newblock In \emph{2019 32nd Symposium on Integrated Circuits and Systems
  Design (SBCCI)}, 1--6.

\bibitem[{Weng et~al.(2020)Weng, Zhao, Liu, Chen, Lin, and
  Daniel}]{model_robustness}
Weng, T.; Zhao, P.; Liu, S.; Chen, P.; Lin, X.; and Daniel, L. 2020.
\newblock Towards Certificated Model Robustness Against Weight Perturbations.
\newblock In \emph{The Thirty-Fourth {AAAI} Conference on Artificial
  Intelligence, {AAAI} 2020, The Thirty-Second Innovative Applications of
  Artificial Intelligence Conference, {IAAI} 2020, The Tenth {AAAI} Symposium
  on Educational Advances in Artificial Intelligence, {EAAI} 2020, New York,
  NY, USA, February 7-12, 2020}, 6356--6363. {AAAI} Press.
\newblock
  \urlprefix\url{https://aaai.org/ojs/index.php/AAAI/article/view/6105}.

\bibitem[{Wiseman and Rush(2016)}]{DBLP:conf/emnlp/WisemanR16}
Wiseman, S.; and Rush, A.~M. 2016.
\newblock Sequence-to-Sequence Learning as Beam-Search Optimization.
\newblock In \emph{Proceedings of the 2016 Conference on Empirical Methods in
  Natural Language Processing, {EMNLP} 2016, Austin, Texas, USA, November 1-4,
  2016}, 1296--1306.

\bibitem[{Yang et~al.(2017)Yang, Hu, Salakhutdinov, and Cohen}]{GAN16}
Yang, Z.; Hu, J.; Salakhutdinov, R.; and Cohen, W.~W. 2017.
\newblock Semi-Supervised {QA} with Generative Domain-Adaptive Nets.
\newblock In \emph{Proceedings of the 55th Annual Meeting of the Association
  for Computational Linguistics, {ACL} 2017, Vancouver, Canada, July 30 -
  August 4, Volume 1: Long Papers}, 1040--1050.
\newblock \doi{10.18653/v1/P17-1096}.

\bibitem[{Yosinski et~al.(2014)Yosinski, Clune, Bengio, and
  Lipson}]{yosinki2014transfer}
Yosinski, J.; Clune, J.; Bengio, Y.; and Lipson, H. 2014.
\newblock How transferable are features in deep neural networks?
\newblock In \emph{Advances in Neural Information Processing Systems 27: Annual
  Conference on Neural Information Processing Systems 2014}, 3320--3328.

\bibitem[{Zhu et~al.(2019)Zhu, Cheng, Gan, Sun, Goldstein, and
  Liu}]{DBLP:journals/corr/freeLB}
Zhu, C.; Cheng, Y.; Gan, Z.; Sun, S.; Goldstein, T.; and Liu, J. 2019.
\newblock FreeLB: Enhanced Adversarial Training for Language Understanding.
\newblock \emph{CoRR} abs/1909.11764.

\end{thebibliography}
\normalsize

\onecolumn
\appendix

\section{Appendix}

\subsection{Theoretical Analysis of the Random Corruption}

\label{sec:TheoreticalRandom}

In this section, we discuss the characteristics of the loss change caused by random corruption under a representative corruption constraint in Theorem~\ref{thm:random}. Here we choose the constraint set  $S=\{\vect{a}:\|\vect{a}\|_2= \epsilon\}$ and show that it is not generally possible for the random corruption to cause substantial loss changes in this circumstance both theoretically and experimentally, making it ineffective in finding vulnerability.

\begin{thmA}[Distribution of Random Corruption]
Given the constraint set  $S=\{\vect{a}:\|\vect{a}\|_2= \epsilon\}$ and a generated random corruption $\vect{\tilde a}$ by the Monte-Carlo estimation, which in turn obeys a uniform distribution on $\|\vect{\tilde a}\|_2=\epsilon$. The first-order estimation of $\Delta_\text{max}\mathcal{L}(\vect{w}, S, \mathcal{D})$ and the expectation of the loss change caused by random corruption is
\begin{align}
\Delta_\text{max}\mathcal{L}(\vect{w}, S, \mathcal{D})=\epsilon G+o(\epsilon), \text{\ and}\quad \mathbb{E}_{\|\vect{\tilde a}\|_2 = \epsilon}[\Delta\mathcal{L}(\vect{w}, \vect{\tilde a}; \mathcal{D})]=O\left(\frac{tr(\textbf{H})}{k}\epsilon^2\right).\label{eq:expection1}
\end{align}

Define $\eta=\nicefrac{|\vect{\tilde a}^\text{T}\vect{g}|}{\epsilon G}$, which is an estimation of  $\nicefrac{|\Delta\mathcal{L}(\vect{w}, \vect{\tilde a}, \mathcal{D})| }{\Delta_\text{max}\mathcal{L}(\vect{w}, S, \mathcal{D})}$ and $\eta\in [0, 1]$, then the probability density function $p_\eta(x)$ of $\eta$ and the cumulative density $P(\eta \le x)$ function of $\eta$ are
\begin{align}
p_{\eta}(x)=\frac{2\Gamma(\frac{k}{2})}{\sqrt{\pi}\Gamma(\frac{k-1}{2})}(1-x^2)^{\frac{k-3}{2}}
,\text{ and}\quad
P(\eta \le x)=\frac{2xF_1(\frac{1}{2}, \frac{3-k}{2};\frac{3}{2}; x^2)}{B(\frac{k-1}{2}, \frac{1}{2})}, \label{equ:random_destiny}
\end{align}
where $k$ denotes the number of corrupted parameters, and $\Gamma(\cdot)$, $B(\cdot,\cdot)$ and $F_1(\cdot,\cdot;\cdot;\cdot)$ denote the gamma function, beta function and hyper-geometric function.
\end{thmA}

The detailed definitions of the gamma function, beta function and hyper-geometric function are as follows: $\Gamma(\cdot)$ and $B(\cdot,\cdot)$ denote the gamma function and beta function, and $F_1(\cdot,\cdot;\cdot;\cdot)$ denotes the Gaussian or ordinary hyper-geometric function, which can also be written as $ _2 F_1(\cdot,\cdot;\cdot;\cdot)$:
\begin{align}
\Gamma(z)&=\int_{0}^{+\infty}t^{z-1}e^{-t}dt,\\  B(p, q)&=\int_{0}^{1}t^{p-1}(1-t)^{q-1}dt,\\
F_1(a,b;c;z)&=1+\sum\limits_{n=1}^{+\infty}\frac{a(a+1)\cdots(a+n-1)\times b(b+1)\cdots(b+n-1)}{c(c+1)\cdots(c+n-1)}\frac{z^n}{n!}.
\end{align}

Theorem~\ref{thm:random} states that the expectation of loss change of random corruption is infinitesimal compared to $\Delta_\text{max}\mathcal{L}(\vect{w}, S, \mathcal{D})$ when $\epsilon\to 0$. In addition, it is unlikely for multiple random trials to induce the optimal loss change corresponding to the indicator. For a deep neural network, the number of parameters, which is the upper bound of $k$, can be considerably large. According to Eq.(\ref{equ:random_destiny}), $\eta$ will be concentrated near $0$. Thus, theoretically, it is not generally possible for the random corruption to cause substantial loss changes in this circumstance, making it ineffective in finding vulnerability. 

The conclusion is also empirically validated. In particular, we define $\alpha(p)$ as a real number in $[0, 1]$ satisfying $P(|\Delta\mathcal{L}(\vect{w}, \vect{\tilde a}, \mathcal{D})|<\alpha(p))=p$, which means $|\Delta\mathcal{L}(\vect{w}, \vect{\tilde a}, \mathcal{D})|<\alpha(p)$ keeps with a probability $p$. We then compare the gradient-based corruption and 10,000 random corruptions on a language model and $\epsilon$ is set to $5\times 10^{-4}$. The distribution of the results of the $10,000$ random corruptions are reported in Table~\ref{tab:probability}, as well as the gradient-based corruption result. We can find that the gradient-based corruption can cause a loss change $\epsilon G$ of $0.044$, while the loss change of the random corruption $|\Delta\mathcal{L}(\vect{w}, \vect{\tilde a}, \mathcal{D})|$ is less than $10^{-4}$ with a high probability, which shows a huge difference of more than 400 times ($\eta < 1/400$) in terms of the corrupting effectiveness.


\begin{table}[H]
\scriptsize
\centering
\vspace{0.1in}
\begin{tabular}{@{}lccc@{}}
\toprule
 Approach & $\alpha(0.9)$ & $\alpha(0.95)$ & $\alpha(0.995)$ \\ \midrule
 Random corruption (Empirical) & $4.0\times 10^{-5}$ &$4.8\times 10^{-5}$ & $7.0 \times 10^{-5}$  \\
 Random corruption (Theoretical) & $1.5\times 10^{-5}$ & $1.8\times 10^{-5}$ & $2.5 \times 10^{-5}$ \\ \midrule
 gradient-based gradient-based corruption & \multicolumn{3}{c}{$ \Delta_\vect{max}\mathcal{L}(\vect{w}, S, \mathcal{D})\approx\epsilon G=0.044$} \\ 
\bottomrule
\end{tabular}
\caption{Probability distribution of corruption effects $|\Delta\mathcal{L}(\vect{w}, \vect{\tilde a}, \mathcal{D})|$ for random corruption. $\alpha(p)$ satisfies $P(|\Delta\mathcal{L}(\vect{w}, \vect{\tilde a}, \mathcal{D})|<\alpha(p))=p$. Random corruption has small loss changes with high probability while gradient-based gradient-based corruption results in a loss change 400 times larger.
}
\label{tab:probability}
\end{table}

\subsection{Proofs}

\subsubsection{Proof of Theorem~\ref{thm:random}}
\begin{proof}
According to Definition~\ref{def:1}, 
\begin{align}
\vect{\tilde a}=\argmax_{\|\vect{a}\|_2=\epsilon}\vect{a}^\text{T}\vect{r}=\epsilon\frac{\vect{r}}{\|\vect{r}\|_2}
\end{align}
where $\vect{r}\sim N(0, 1)\ (1\le i\le k)$. Note that $\vect{r}$ obeys the Gaussian distribution with a mean vector of zero and a covariance matrix of $I$. Thus, the distribution of $\vect{r}$ has rotational invariance and $\vect{\tilde  a}=\epsilon\frac{\vect{r}}{\|\vect{r}\|_2}$ also has rotational invariance. Therefore, $\vect{\tilde  a}$ obeys a uniform distribution on $\|\vect{\tilde a}\|_2=\epsilon$.

First, We will prove Eq.(\ref{eq:expection1}).
\begin{align}
\Delta\mathcal{L}(\vect{w},\vect{a}; \mathcal{D}) &= \vect{a}^\text{T}\vect{g}+\frac{1}{2}\vect{a}^\text{T}\textbf{H}\vect{a}+o(\epsilon^2)=f(\vect{a})+o(\epsilon), \\
\max\limits_{\|\vect{a}\|_2=\epsilon} f(\vect{a})&=\max\limits_{\|\vect{a}\|_2=\epsilon} \vect{a}^\text{T}\vect{g}=\epsilon G
\end{align}

Therefore,
\begin{align}
\Delta_\text{max}\mathcal{L}(\vect{w}, S, \mathcal{D})=\epsilon G+o(\epsilon)
\end{align}

Suppose $\vect{\tilde a}=(a_1, a_2, \cdots,a_{k-1}, a_{k})^\text{T}, \vect{g}=(g_1, g_2, \cdots,g_{k-1}, g_{k})^\text{T}$ and  $H_{ij}=\nicefrac{\partial^2\mathcal{L}(\vect{w}+\vect{a};\mathcal{D})}{\partial a_i \partial a_j}$. Since  $\vect{\tilde  a}$ obeys a uniform distribution on $\|\vect{\tilde a}\|_2=\epsilon$, by symmetry, we have, 
\begin{align}
\mathbb{E}_{\|\vect{\tilde a}\|_2=\epsilon}[a_i]&=
\mathbb{E}_{\|\vect{\tilde a}\|_2=\epsilon}[a_ia_j]=0\ (i\ne j)\\
\mathbb{E}_{\|\vect{\tilde a}\|_2=\epsilon}[a_i^2]&=
\mathbb{E}_{\|\vect{\tilde a}\|_2=\epsilon}[\frac{\|\vect{a}\|^2}{k}]=\frac{\epsilon^2}{k}
\end{align}

Therefore,
\begin{align}
\mathbb{E}_{\|\vect{\tilde a}\|_2=\epsilon}[\Delta\mathcal{L}(\vect{w}, \vect{\tilde a}, \mathcal{D})] &= \mathbb{E}_{\|\vect{\tilde a}\|_2=\epsilon}[\vect{\tilde a}^\text{T}\vect{g}+\frac{1}{2}\vect{\tilde a}^\text{T}\textbf{H}\vect{\tilde a}+o(\epsilon^2)] \\
&=\mathbb{E}_{\|\vect{\tilde a}\|_2=\epsilon}[\vect{\tilde a}^\text{T}\vect{g}]+\mathbb{E}_{\|\vect{\tilde a}\|_2=\epsilon}[\frac{1}{2}\vect{\tilde a}^\text{T}\textbf{H}\vect{\tilde a}]+o(\epsilon^2)\\
&=\mathbb{E}_{\|\vect{\tilde a}\|_2=\epsilon}[\sum\limits_{i}g_ia_i]+\mathbb{E}_{\|\vect{\tilde a}\|_2=\epsilon}[\frac{1}{2}\sum\limits_{i, j}H_{ij}a_ia_j]+o(\epsilon^2)\\
&=\sum\limits_{i}H_{ii}\frac{\epsilon^2}{2k}+o(\epsilon^2)\\
&=O(\frac{tr(\textbf{H})}{k}\epsilon^2)
\end{align}

Then, we will prove Eq.(\ref{equ:random_destiny}). Because of the rotational invariance of the distribution of $\vect{\tilde  a}$, we may assume $\frac{\vect{g}}{\|\vect{g}\|_2}=(1, 0, 0, \cdots, 0)^\text{T}, \vect{\tilde  a}=(a_1, a_2, a_3,\cdots,a_{k-1}, a_{k})^\text{T}$ and
\begin{equation} 
\left \{
\begin{aligned} 
a_1 &= \epsilon\cos\phi_1\\ 
a_2 &= \epsilon\sin\phi_1\cos\phi_2\\
a_3 &= \epsilon\sin\phi_1\sin\phi_2\cos\phi_3\\
&\quad \cdots\\
a_{k-1} &= \epsilon\sin\phi_1\sin\phi_2\cdots\sin\phi_{k-2}\cos\phi_{k-1}\\
a_{k} &= \epsilon\sin\phi_1\sin\phi_2\cdots\sin\phi_{k-2}\sin\phi_{k-1}\\
\end{aligned} 
\right. 
\end{equation}
where $\phi_i\in[0, \pi]\ (i\ne k-1)$ and $\phi_{k-1}\in[0, 2\pi)$. For $x \in [0, 1]$, define $\alpha = \arccos x$, then:
\begin{equation}
f(\vect{\tilde  a})=\vect{\tilde  a}^\text{T}\vect{g}=\epsilon G\cos\phi_1, P(\eta \le x)=P(|\cos\phi_1| \le x)=2P(0\le\phi_1\le \alpha)
\end{equation}

That is to say,
\begin{align}
P(\eta\le x)&=\frac{2\int_0^{2\pi}\int_0^{\pi}\cdots\int_0^{\alpha}(\sin^{k-2}\phi_{1}\sin^{k-3}\phi_{2}\cdots \sin\phi_{k-2}) d\phi_1\cdots  d\phi_{k-2}d\phi_{k-1}}{\int_0^{2\pi}\int_0^{\pi}\cdots\int_0^{\pi}(\sin^{k-2}\phi_{1}\sin^{k-3}\phi_{2}\cdots \sin\phi_{k-2}) d\phi_1\cdots  d\phi_{k-2}d\phi_{k-1}}\\
&=\frac{2\int_0^{\alpha}\sin^{k-2}\phi_{1}d\phi_1}{\int_0^{\pi}\sin^{k-2}\phi_{1}d\phi_1}=\frac{\int_0^{\alpha}\sin^{k-2}\phi_{1}d\phi_1}{\int_0^{\frac{\pi}{2}}\sin^{k-2}\phi_{1}d\phi_1}=\frac{2\int_0^{\alpha}\sin^{k-2}\phi_{1}d\phi_1}{B(\frac{k-1}{2}, \frac{1}{2})} \label{equ:P}\\
&=\frac{2\cos\alpha F_1(\frac{1}{2}, \frac{3-k}{2};\frac{3}{2}; \cos^2\alpha)}{B(\frac{k-1}{2}, \frac{1}{2})}=\frac{2xF_1(\frac{1}{2}, \frac{3-k}{2};\frac{3}{2}; x^2)}{B(\frac{k-1}{2}, \frac{1}{2})}
\end{align}
and notice that
\begin{align}
\sin\alpha=(1-x^2)^\frac{1}{2}, \big|\frac{d\alpha}{dx}\big|=\frac{1}{(1-x^2)^\frac{1}{2}}, B(p, q)=\frac{\Gamma(p)\Gamma(q)}{\Gamma(p+q)}, \Gamma(\frac{1}{2})=\sqrt{\pi}
\end{align}
then according to Eq.(\ref{equ:P}):
\begin{align}
p_\eta(x)=\frac{2\sin^{k-2}\alpha}{B(\frac{k-1}{2}, \frac{1}{2})}\big|\frac{d\alpha}{dx}\big|=\frac{2\Gamma(\frac{k}{2})}{\sqrt{\pi}\Gamma(\frac{k-1}{2})}(1-x^2)^{\frac{k-3}{2}}
\end{align}
\end{proof}

\subsubsection{Closed-Form Solutions in Definitions}

The close-form solutions of the random parameter corruption in Definition~\ref{def:1} and gradient-based corruption in Definition~\ref{def:2} can be generalized into Proposition~\ref{prop:linear}, which is the maximum of linear function under the corruption constraint.

\begin{prop}[Constrained Maximum]
\label{prop:linear}
Given a vector $\vect{v}\in \mathbb{R}^k$, the optimal $\vect{\hat a}$ that maximizes $\vect{a}^\text{T}\vect{v}$ under the corruption constraint $\vect{a}\in S=\{\vect{a}:\|\vect{a}\|_p= \epsilon\text{ and }\|\vect{a}\|_0\le n\}$ is
\begin{equation}
\vect{\hat a}=\argmax_{\vect{a}\in S}\vect{a}^\text{T}\vect{v}=\epsilon (\text{sgn}(\vect{h})\odot\frac{|\vect{h}|^\frac{1}{p-1}}{\||\vect{h}|^\frac{1}{p-1}\|_p}),\text{ and}\quad \vect{\hat a}^\text{T}\vect{v}=\epsilon\|\vect{h}\|_{\frac{p}{p-1}},
\end{equation}
where $\vect{h}=\text{top}_n(\vect{v})$, retaining top-$n$ magnitude of all $|\vect{v}|$ dimensions and set other dimensions to $0$, $\text{sgn}(\cdot)$ denotes the signum function, $|\cdot|$ denotes the point-wise absolute function, and $(\cdot)^\alpha$ denotes the point-wise $\alpha$-power function.
\end{prop}

\begin{proof}
When $\vect{a}\in S=\{\vect{a}:\|\vect{a}\|_p=\epsilon\text{ and }\|\vect{a}\|_0\le n\}$, define $\vect{a}=\textbf{P}\vect{b}$, where $\textbf{P}$ is a diagonal $0/1$ matrix with $n$ ones. It is easy to verify $\textbf{P}^\text{T}=\textbf{P}=\textbf{P}^2$. Define $q=\frac{p}{p-1}, \frac{1}{p}+\frac{1}{q}=1$ here. Then according to Holder Inequality, for $\frac{1}{p}+\frac{1}{q}=1,(1\le p, q\le+\infty)$, 
\begin{align}
\vect{a}^\text{T}\vect{v}=\vect{b}^\text{T}\textbf{P}\vect{v}=\vect{b}^\text{T}\textbf{P}\textbf{P}\vect{v}=\vect{a}^\text{T}(\textbf{P}\vect{v})\le\|\vect{a}\|_p\|\textbf{P}\vect{v}\|_q=\epsilon\|\vect{h}\|_{\frac{p}{p-1}}
\end{align}
where $\vect{h}=\textbf{M}\vect{v}=\text{top}_n(\vect{v})$, $\textbf{M}$ is a diagonal $0/1$ matrix and $M_{j,j}=1$ if and only if $|\vect{v}|_j$ is in the top-$n$ magnitude of all $|\vect{v}|$ dimensions. The equation holds if and only if,
\begin{align}
\vect{\hat a}=\epsilon(\text{sgn}(\vect{h})\odot\frac{|\vect{h}|^{\frac{1}{p-1}}}{\||\vect{h}|^{\frac{1}{p-1}}\|_p})
\end{align}
and the maximum value of $\vect{a}^\text{T}\vect{v}$ is $\vect{\hat a}^\text{T}\vect{v}=\epsilon\|\vect{h}\|_{\frac{p}{p-1}}$.
\end{proof}

Proposition~\ref{prop:linear} indicates that the maximization of the value $\vect{a}^\text{T}\vect{v}$ under the corruption constraint has a closed-form solution. The solutions to a special case, where $p=+\infty$ are shown below.

\begin{cor}
\label{prop:inf}
When $p=+\infty$, the solution in Proposition~\ref{prop:linear} is,
\begin{equation}
\vect{\hat a}=\lim\limits_{p\to +\infty}\epsilon (\text{sgn}(\vect{h})\odot\frac{|\vect{h}|^\frac{1}{p-1}}{\||\vect{h}|^\frac{1}{p-1}\|_p})=\epsilon\text{sgn}(\vect{h}),\text{\ and}\quad \vect{\hat a}^\text{T}\vect{v}=\epsilon\|\vect{h}\|_1
\end{equation}
\end{cor}

\begin{proof}
In Proposition~\ref{prop:linear}, when $p\to +\infty$, $0^{\frac{1}{p-1}}\to 0, x^{\frac{1}{p-1}}\to 1\ (x\ne 0)$ and $|\vect{h}|^{\frac{1}{p-1}}\to \mathbb{I}(\vect{h}\ne 0)$. Then 
\begin{align}
\vect{\hat a}=\lim\limits_{p\to +\infty }\epsilon (\text{sgn}(\vect{h})\odot\frac{|\vect{h}|^{\frac{1}{p-1}}}{\||\vect{h}|^{\frac{1}{p-1}}\|_p})=\epsilon \text{sgn}(\vect{h})
\end{align}
and the maximum value of $\vect{a}^\text{T}\vect{v}$ is $\vect{\hat a}^\text{T}\vect{v}=\epsilon\|\vect{h}\|_{\frac{p}{p-1}}=\epsilon\|\vect{h}\|_{1}$.

\end{proof}

\subsubsection{Proof of Theorem~\ref{thm:bound}}

\begin{thmA}[Error Bound of the Gradient-based Estimation]
Suppose $\mathcal{L}(\vect{w};\mathcal{D})$ is convex and $L$-smooth with respect to $\vect{w}$ in the subspace $\{\vect{w}+\vect{a}:\vect{a}\in S\}$, where $S=\{\vect{a}:\|\vect{a}\|_p=\epsilon\text{ and }\|\vect{a}\|_0\le n\}$.\footnote{Note that $\mathcal{L}$ is only required to be convex and $L$-smooth in a neighbourhood of $\vect{w}$, instead of the entire $\mathbb{R}^k$.} Suppose $\vect{a^*}$ and $\vect{\hat a}$ are the optimal corruption and the gradient-based corruption in $S$ respectively. $\|\vect{g}\|_2=G>0$. It is easy to verify that $\mathcal{L}(\vect{w}+\vect{a^*};\mathcal{D})\ge \mathcal{L}(\vect{w+\vect{\hat a}};\mathcal{D})>\mathcal{L}(\vect{w};\mathcal{D})$ . It can be proved that the loss change of gradient-based corruption is the same order infinitesimal of that of the optimal parameter corruption:
\begin{equation}
\frac{\Delta_\text{max}\mathcal{L}(\vect{w}, S; \mathcal{D})}{\Delta\mathcal{L}(\vect{w}, \vect{\hat a}; \mathcal{D})}=1+O\left(\frac{Ln^{g(p)}\sqrt{k}\epsilon}{G}\right);
\label{eq:bound1}
\end{equation}
where $g(p)$ is formulated as $g(p)=\max\{\frac{p-4}{2p}, \frac{1-p}{p}\}$.
\end{thmA}

Theorem~\ref{thm:bound} guarantees when perturbations to model parameters are small enough, the gradient-based corruption can accurately estimate the indicator with small errors. In Eq.(\ref{eq:bound1}), the numerator is the proposed indicator, which is the maximum loss change caused by parameter corruption, and the denominator is the loss change with the parameter corruption generated by the gradient-based method. As we can see, when  $\epsilon$, the $p$-norm of the corruption vector, tends to zero, the term $O(\cdot)$ will also tend to zero such that the 
ratio becomes one, meaning the gradient-based method is an infinitesimal estimation of the indicator.

\begin{proof}
Define $q=\frac{p}{p-1}, \frac{1}{p}+\frac{1}{q}=1$ here. 

We introduce a lemma first.
\begin{lem}
For vector $x\in \mathbb{R}^k$, $\|\vect{x}\|_0\le n \le k$, for any $r>1$, $\|\vect{x}\|_2\le \beta_r \|\vect{x}\|_r$, where $\beta_r=\max\{1, n^{1/2-1/r}\}$. 
\label{lemma:norm}
\end{lem}
\begin{proof}[Proof of Lemma~\ref{lemma:norm}]
We may assume $\vect{x}=(x_1, x_2, \cdots, x_k)^\text{T}$ and $x_{n+1}=x_{n+2}=\cdots=x_{k}=0$. Then $\|\vect{x}\|_r=\big(\sum\limits_{i=1}^n|x_i|^r\big)^{\frac{1}{r}}$.

When $1<r<2$, define $t=\frac{r}{2}<1$ and $h(x)=x^t+(1-x)^t$, $h''(x)=t(t-1)(x^{t-2}+(1-x)^{t-2})<0$, thus $h(x)\ge \max\{h(0), h(1)\}=1\ (x\in[0, 1])$. 

Then for $a, b\ge 0$ and $a+b>0$, we have $\frac{a^t+b^t}{(a+b)^t}=(\frac{a}{a+b})^t+(1-\frac{a}{a+b})^t=h(\frac{a}{a+b})\ge 1$. That is to say, $a^t+b^t \ge (a+b)^t$. More generally, $a^t+b^t+\cdots+c^t \ge (a+b+\cdots+c)^t$. Therefore,
\begin{align}
\|\vect{x}\|_r=\big(\sum\limits_{i=1}^n|x_i|^{r}\big)^{\frac{1}{r}}=\big(\sum\limits_{i=1}^n(|x_i|^2)^\frac{r}{2}\big)^{\frac{1}{r}} \ge \big((\sum\limits_{i=1}^n|x_i|^2)^\frac{r}{2}\big)^{\frac{1}{r}} = \|\vect{x}\|_2
\end{align}

When $r\ge 2$, according to the power mean inequality,
\begin{align}
\|\vect{x}\|_r=\big(\sum\limits_{i=1}^n|x_i|^{r}\big)^{\frac{1}{r}}=n^\frac{1}{r} \big(\frac{\sum\limits_{i=1}^n|x_i|^r}{n}\big)^{\frac{1}{r}} \ge n^\frac{1}{r} \big(\frac{\sum\limits_{i=1}^n|x_i|^2}{n}\big)^{\frac{1}{2}}=  n^{\frac{1}{r}-\frac{1}{2}}(\sum\limits_{i=1}^n|x_i|^2)^\frac{1}{2} = n^{\frac{1}{r}-\frac{1}{2}}\|\vect{x}\|_2
\end{align}

To conclude, $\|\vect{x}\|_2\le \beta_r \|\vect{x}\|_r$, where $\beta_r=\max\{1, n^{1/2-1/r}\}$.
\end{proof}

According to Lemma~\ref{lemma:norm}, notice that $\|\vect{a}^*\|_0\le n$, define $\vect{h}=\text{top}_n(\vect{g})$, then  $\|\vect{h}\|_2\ge\frac{n}{k}\|\vect{g}\|_2$ we have, 
\begin{align}
\|\vect{a}^*\|_2 \le \beta_p\|\vect{a}^*\|_p\le \beta_p\epsilon
,\quad 
\|\vect{h}\|_q\ge\frac{\|\vect{h}\|_2}{\beta_q}\ge \frac{\|\vect{g}\|_2}{\beta_q}\sqrt{\frac{n}{k}}= \frac{G}{\beta_q}\sqrt{\frac{n}{k}}
\end{align}

Since $\mathcal{L}(\vect{w};\mathcal{D})$ is convex and $L$-smooth in $\vect{w}+S$,
\begin{align}
\Delta\mathcal{L}(\vect{w}, \vect{\hat a}, \mathcal{D}))&\ge \vect{g}^\text{T}\vect{\hat a}=\epsilon \|\vect{h}\|_q\\
\Delta\mathcal{L}(\vect{w}, \vect{a^*}, \mathcal{D})&\le \vect{g}^\text{T}\vect{a}^*+\frac{L}{2}\|\vect{a}^*\|_2^2 = \epsilon \|\vect{h}\|_q+\frac{L}{2}\|\vect{a}^*\|_2^2 \end{align}

Therefore,
\begin{align}
\text{Left Hand Side}=\frac{\Delta\mathcal{L}(\vect{w}, \vect{a^*}, \mathcal{D})}{\Delta\mathcal{L}(\vect{w}, \vect{\hat a}, \mathcal{D})}
\le\frac{\epsilon \|\vect{h}\|_q+\frac{L}{2}\|\vect{a}^*\|_2^2}{\epsilon \|\vect{h}\|_q}
\le1+\frac{L\beta_p^2\epsilon}{2\|\vect{h}\|_q}
\le1+\frac{L\beta_p^2\beta_q\epsilon\sqrt{k}}{2G\sqrt{n}}
\end{align}

When $p\ge 2,q\le 2$,  $\beta_p^2\beta_q=n^{1-2/p}$, and when $p\le 2,q\ge 2$, $\beta_p^2\beta_q=n^{1/2-1/q}=n^{1/p-1/2}$. To conclude, $\beta_p^2\beta_q=\max\{n^{1-2/p}, n^{1/p-1/2}\}=n^{\max\{1-2/p, 1/p-1/2\}}$. Therefore,
\begin{align}
\text{Left Hand Side}\le 1+\frac{Ln^{\max\{1-2/p, 1/p-1/2\}}\sqrt{k}}{2G\sqrt{n}}\epsilon = 1+O\left(\frac{Ln^{g(p)}\sqrt{k}\epsilon}{G}\right)
\end{align}
where $g(p)=\max\{\frac{p-4}{2p}, \frac{1-p}{p}\}$.
\end{proof}

\subsection{Computational Complexity of Parameter Corruption}
For the parameter corruption, the closed-form solution is determined by: (1) the set of allowed corruptions $S$, i.e., $n, k, \epsilon$ and $p$; (2) the vector $\vect{v}$, corresponding to $\vect{r}$ in random corruption and $\vect{g}$ in gradient-based corruption. 
The first condition has a computational cost of $O(k \log n)$, because a top-$n$ check is involved, which is at least $O(k \log n)$, and other calculations, such as multiplication and inversion, are less than $O(k)$. 
For the second condition, sampling $\vect{r}$ should be trivial on modern computers but obtaining $\vect{g}$ with respect to the whole dataset can be costly, generally of cost $O(k|\mathcal{D}|)$. 

\subsection{Model Implementation}
This section shows the implementation details of neural networks used in our experiments. Experiments are conducted on a GeForce GTX TITAN X GPU.

\subsubsection{ResNet}
We adopt Resnet-34~\cite{he2016deep} as the baseline on two benchmark datasets: CIFAR-10~\cite{torralba200880}  and ImageNet~\cite{imagenet}. CIFAR-10\footnote{CIFAR-10 can be found at \url{https://www.cs.toronto.edu/~kriz/cifar.html}} is an image classification dataset with 10 categories and consists of 50000 training images and 10000 test images. The images are of 32-by-32 pixel size with 3 channels. We adopt the classification accuracy as our evaluation metric on CIFAR-10. For ImageNet\footnote{ImageNet can be found at \url{http://www.image-net.org}}, it contains $10M$ labeled images of more than $10k$ categories as the training dataset and $50k$ images as the validation dataset. Note that the gradient-based parameter corruption on ResNet on ImageNet uses gradients from the validation set due to the sheer size of ImageNet, while other experiments use gradients from the training set. We adopt the accuracy (Acc) on the validation dataset as our evaluation metrics on ImageNet.
We adopt SGD as the optimizer with a mini-batch size of 128. The learning rate is 0.1; the momentum is 0.9; and the weight-decay is $5\times 10^{-4}$. We train the model for $200$ epochs. We also apply data augmentation for training following~\cite{he2016deep}: 4 pixels are padded on each side, and a $32 \times 32$ crop is randomly sampled from the padded image or its horizontal flip.

\subsubsection{Transformer}
We implement the ``transformer\_iwslt\_de\_en'' provided by fairseq~\cite{ott2019fairseq} on German-English translation dataset. We use the same dataset splits following \citet{ott2019fairseq,DBLP:journals/corr/RanzatoCAZ15,DBLP:conf/emnlp/WisemanR16}. The dataset and the implementation of ``transformer\_iwslt\_de\_en'' can be found at the fairseq repository\footnote{\url{https://github.com/pytorch/fairseq}}. The dataset contains 160K sentences for training, 7K sentences for validation, and 7K sentences for testing. BPE is used to get the vocabulary. We adopt the BLEU score as the evaluation metric on the translation task. The dropout rate is $0.4$. The training batch size is $4000$ tokens and we update the model for every $4$ step. We adopt a learning rate schedule with an initial learning rate of $10^{-7}$ and a base learning rate of $0.001$. The weight decay rate is $0.0001$. The number of warmup steps is $4000$. We set the test beam-size to $5$. We train the model for $70$ epochs. We adopt the checkpoint-average mechanism for evaluation and the last $10$ checkpoints are averaged. 

\subsubsection{LSTM} 
We use a 3-layer LSTM as a language model following \cite{merityRegOpt,merityAnalysis} on the word-level Penn TreeBank dataset (PTB)\footnote{PTB can be found at \url{https://www.kaggle.com/nltkdata/penn-tree-bank?select=ptb}}\citep{DBLP:journals/coling/MarcusSM94}. The dataset has been preprocessed and the vocabulary size is limited to $10000$ words. We adopt the log perplexity on the test set as the evaluation metric on PTB.
We use the Adam optimizer and initialize the learning rate with $0.001$. The embedding size is set to $400$ and the hidden size is set to $1150$. We train the model for $70$ epochs. We use a weight-decay of $1.2*10^{-6}$. 

\subsubsection{MLP}
We implement a MLP-based parser following \citet{DBLP:conf/emnlp/ChenM14} on a transition-based dependency parsing dataset provided by
the English Penn TreeBank (PTB)~\citep{DBLP:journals/coling/MarcusSM94}. We follow the standard splits of the dataset and adopt Unlabeled Attachment Score (UAS) as the evaluation metric.
The hidden size is set to $50$ and the batch size is $1024$. The word embeddings are initialized as a $50$-dim GloVe~\citep{glove} embeddings. We use the Adam optimizer with an initial learning rate of $0.001$ and a learning rate decay of $0.9$. We train the model for $20$ epochs. We evaluate the model on the development set every epoch and find the best checkpoint to evaluate the test results. The dropout rate is $0.2$.

\subsection{Detailed Experimental Settings and Supplemental Results of the Proposed Parameter Corruption Approach}

On PTB-Parsing, we adopt the $L_2$ gradient regularization and $\lambda=0.01$. On Tiny-ImageNet, CIFAR-10, PTB-LM and De-En, we adopt the variant of $\mathcal{L}_a$ by directly optimizing $\mathcal{L}^*$, instead of adopting the gradient regularization term. We set $\alpha=0.5$ and choose different $\epsilon$ to control the magnitude of $\vect{\hat a}$. We set $\alpha=0.5$ and choose different $\epsilon$ to control the magnitude of $\vect{\hat a}$. On Tiny-ImageNet, we choose $L_{+\infty}$-norm and $\epsilon=0.00002$. On CIFAR-10, we choose $L_2$-norm and $\epsilon=0.5$. On PTB-LM, we choose $L_{+\infty}$-norm and $\epsilon=0.1$. On De-En, we choose $L_{+\infty}$-norm and $\epsilon=0.0006$ and we start corruption-resistant training after $30$ epochs. On PTB-Parsing, we adopt $L_2$ gradient regularization and $\lambda=0.01$. Other settings are the same as the baseline models.

In Figure~\ref{fig:def}, we show the influence of the magnitude of the virtual gradient-based corruption to the accuracy.

\subsection{Supplemental Experimental Results of the Proposed Parameter Corruption Approach}

In this section, we report supplemental experimental results of our proposed parameter corruption approach to analyze and visualize the weak points of selected deep neural networks, which illustrates the divergent vulnerability of different components within a neural network.

\subsubsection{Vulnerability in Terms of Parameter Positions}

In our paper, we visualize the influence of parameter positions on the vulnerability of deep neural networks by applying the gradient-based gradient-based corruption consecutively to layers in ResNet-34 and Transformer. In this section, we report some experimental results in Figure~\ref{fig:Layer_attack1}, Table~\ref{tab:results} and Table~\ref{tab:results_layer} as the supplemental experimental results of our work. Several observations can be drawn: 1) Lower layers in ResNet-34 are less robust to parameter corruption and are prone to cause overall damage; 2) Upper layers in Transformer decoder are less robust to parameter corruption while the parameter corruption exhibits inconspicuous trend for Transformer encoder.

\begin{figure}[H]
\centering
\subcaptionbox{ResNet-34 (ImageNet)}{\includegraphics[width=0.22\linewidth]{pics/2/imagenet_hot.pdf}}
\hfil
\subcaptionbox{ResNet-34 (CIFAR-10)}{\includegraphics[width=0.22\linewidth]{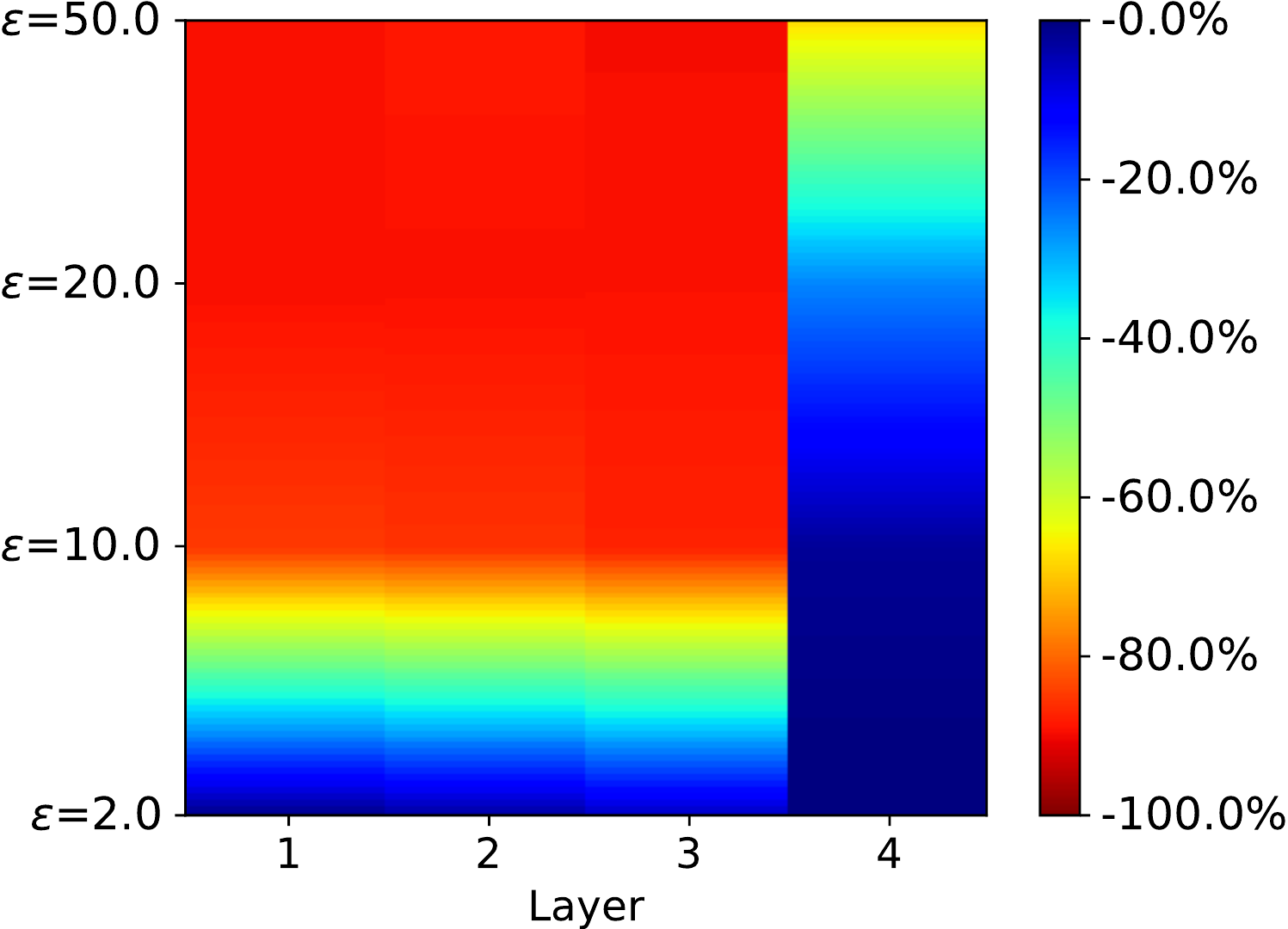}}
\hfil
\subcaptionbox{Transformer Decoder}{\includegraphics[width=0.22\linewidth]{pics/2/fairseq_decoder_hot.pdf}}
\hfil
\subcaptionbox{Transformer Encoder}{\includegraphics[width=0.22\linewidth]{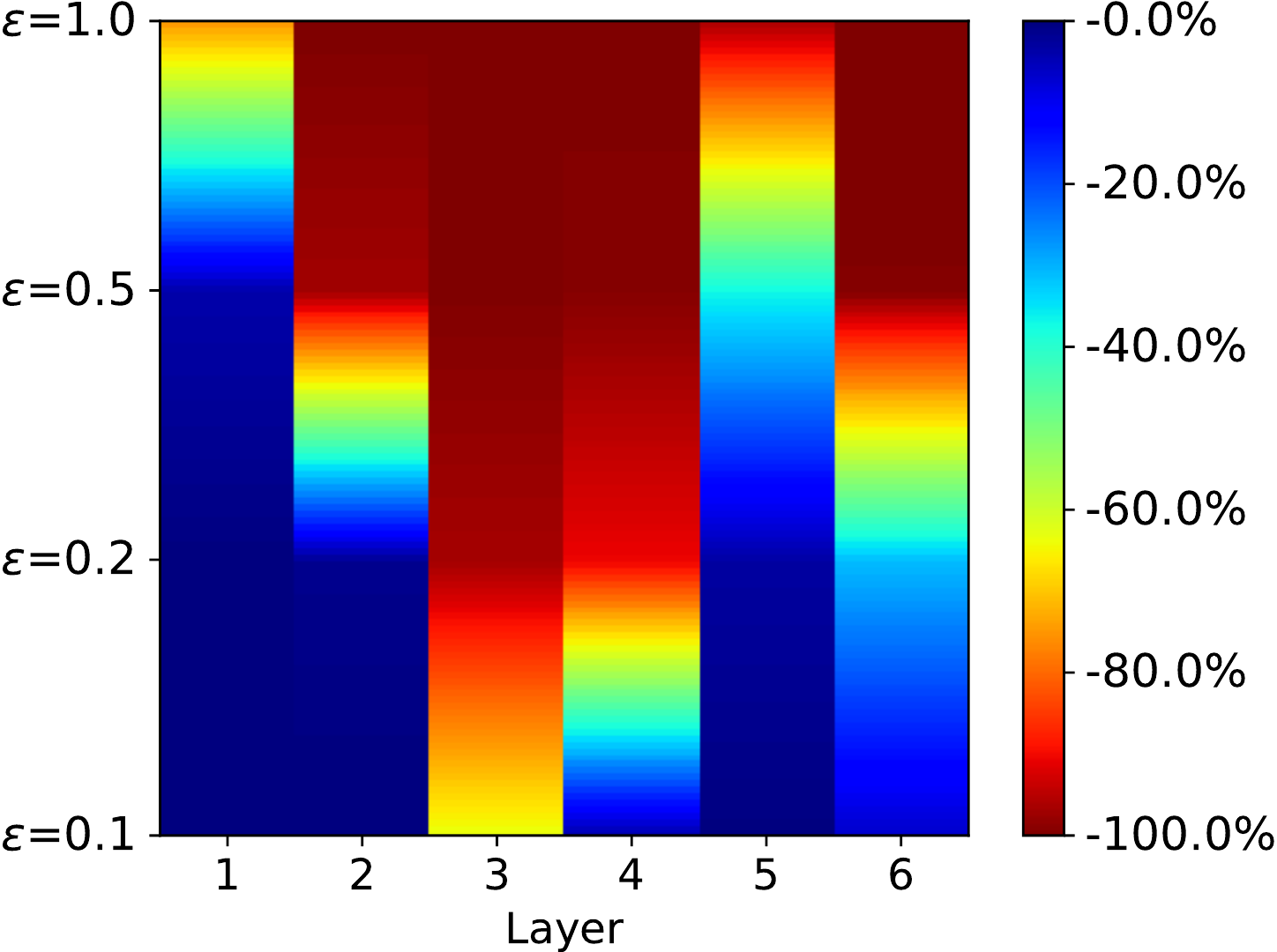}}
\caption{
Results of gradient-based gradient-based corruption on different layers in ResNet-34 and Transformer ($n=100$, $L_{+\infty}$-norm). Warmer colors indicate significant performance degradation. For example, ``red'' means the most significant performance degradation.  Results show that different neural networks exhibit diverse characteristics. Particularly, lower layers in ResNet-34 are more robust to parameter corruption while Transformer decoder shows the opposite trend. 
}
\label{fig:Layer_attack1}
\end{figure}

\begin{table}[H]
\scriptsize
\setlength{\tabcolsep}{4pt}
\centering
\begin{tabular}{llccccllcccc}
\toprule
 \textbf{CIFAR-10} & Layer1 & Layer2 & Layer3 & Layer4 & \textbf{ImageNet} & Layer1 & Layer2 & Layer3 & Layer4 \\
\cmidrule(r){1-5}\cmidrule{6-10}
 w/o Corruption & \multicolumn{4}{c}{94.3 $\star$} & w/o Corruption & \multicolumn{4}{c}{72.5 $\star$} \\
 \cmidrule(r){1-5}\cmidrule{6-10}
$\epsilon=10$ & $\star$ & $\star$ & $\star$ & $\star$ &$\epsilon=0.1$ & 72.2 & $\star$ & 72.2 & $\star$ \\
$\epsilon=20$ & 93.6 & $\star$ & 87.6 & $\star$ &$\epsilon=0.2$ & 71.7 & 72.1 & 71.7 & $\star$\\
$\epsilon=50$ & 62.2 & 29.6 & 31.9 & $\star$ &$\epsilon=0.5$ & 71.0 & 70.5 & 69.9 & 71.3\\
$\epsilon=100$ & 13.5 & 13.2 & 11.8 & 92.5 &$\epsilon=1$ & 66.5 & 68.5 & 67.9 & 69.3\\
$\epsilon=200$ & 10.0 & 10.1 & 10.3 & 71.4&$\epsilon=2$ & 51.6 & 56.9 & 63.0 & 64.2 \\
$\epsilon=500$ & 10.0 & 10.9 & 9.8 & 30.7 &$\epsilon=5$ & 2.5 & 1.4 & 1.4 & 27.2 \\
\bottomrule
\end{tabular}
\caption{Results of corrupting different layers ($n=100$, $L_2$-norm) of ResNet-34 on ImageNet and CIFAR-10. $\star$ denotes original performance scores without parameter corruption and scores close to the original score (difference less than 0.1).}
\label{tab:results}
\end{table}

\begin{minipage}[H]{\textwidth}
\begin{minipage}[l]{0.6\textwidth}
\begin{table}[H]
\scriptsize
\setlength{\tabcolsep}{6pt}
\centering
\begin{tabular}{llcccccc}
\toprule
& & Layer1 & Layer2 & Layer3 & Layer4 & Layer5 & Layer6 \\
\midrule
 & w/o Corruption & \multicolumn{6}{c}{35.33 $\star$} \\
 \midrule
\multirow{5}{*}{Encoder} & $\epsilon=0.2$ & $\star$ & $\star$ & $\star$ & $\star$ & $\star$ & 34.88 \\
& $\epsilon=0.5$ & $\star$ & $\star$ & 29.89 & 35.06 & $\star$ & 32.06 \\
& $\epsilon=1$ & $\star$ & $\star$ & 9.68 & 33.54 & 35.04 & 1.76 \\
& $\epsilon=2$ & $\star$ & 27.23 & 0.53 & 2.79 & 32.85 & 0.05 \\
& $\epsilon=5$ & 33.21 & 0.44 & 0.07 & 0.16 & 9.28 & 0 \\
 \midrule
\multirow{5}{*}{Decoder} & $\epsilon=0.2$ & $\star$ & $\star$ & $\star$ & $\star$ & $\star$ & 34.40 \\
& $\epsilon=0.5$ & 35.17 & 34.97 & $\star$ & 35.21 & 35.07 & 31.02 \\
& $\epsilon=1$ & 34.99 & 34.71 & 35.17 & 34.53 & 34.37 & 21.90 \\
& $\epsilon=2$ & 33.86 & 33.08 & 32.09 & 31.00 & 32.00 & 1.97 \\
& $\epsilon=5$ & 7.96 & 0.11 & 0.86 & 1.72 & 3.18 & 0.02\\
\bottomrule
\end{tabular}
\caption{Results of corrupting different layers ($n=100$, $L_2$-norm) of Transformer on De-En. $\star$ denotes original performance scores without parameter attack and scores close to the original score (difference less than 0.1).}
\label{tab:results_layer}
\end{table}
\end{minipage}
\quad
\begin{minipage}[r]{0.35\textwidth}
\begin{figure}[H]
    \includegraphics[scale=.4]{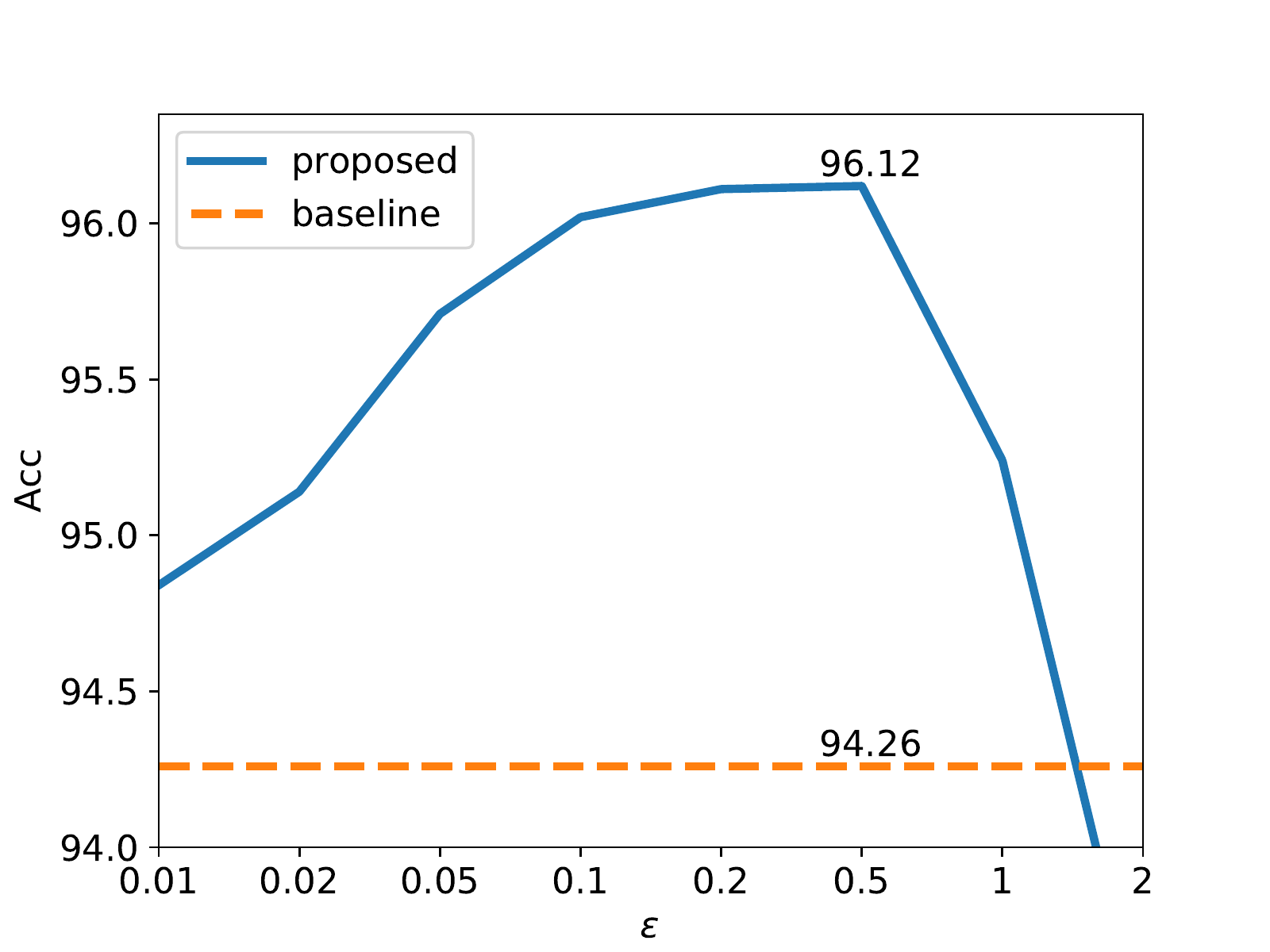}
    \caption{Results of the corruption-resistant training with different $\epsilon$ on CIFAR-10.}
    \label{fig:def}
\end{figure}
\end{minipage}
\end{minipage}

\subsubsection{Detailed Visualization of the Vulnerability of Different Components in ResNet-34}

In our paper, ResNet-34 is roughly divided into four layers. In this section, we report detailed visualization of the vulnerability of different components in ResNet-34 as shown in Figure~\ref{fig:resnet_details}. We can see that 1) Lower layers in ResNet-34 are more robust to parameter corruption; 2) Batch normalization in ResNet-34 is usually less robust to parameter corruption compared to its neighborhood components.

\clearpage
\begin{figure*}[ht]
\centering
\subcaptionbox{ResNet-34 (ImageNet).}{
\begin{minipage}[t]{0.48\linewidth}
\centering
\includegraphics[height=8 in, width=2.2 in]{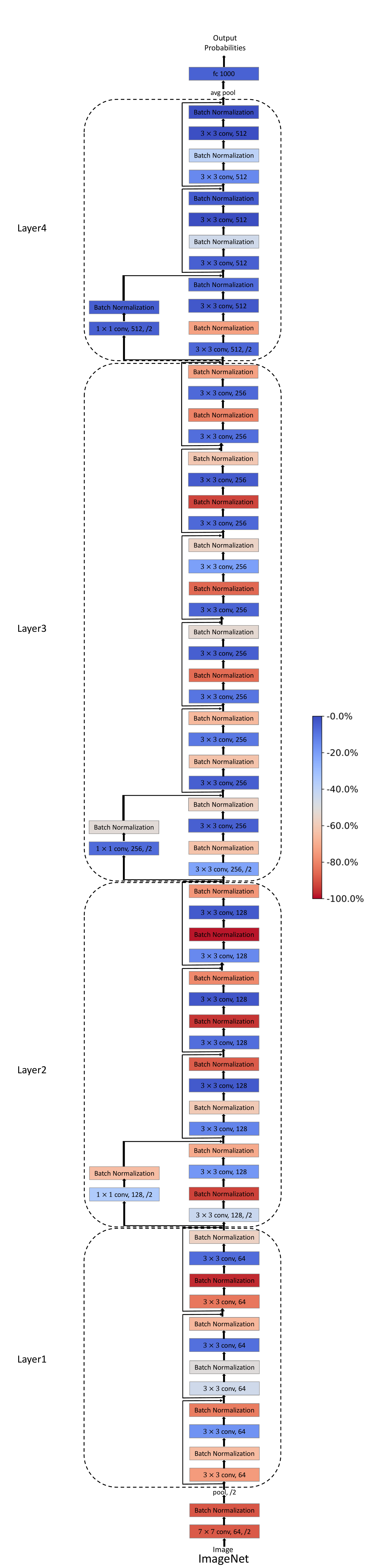}
\end{minipage}
}%
\subcaptionbox{ResNet-34 (CIFAR-10).}{
\begin{minipage}[t]{0.48\linewidth}
\centering
\includegraphics[height=8 in, width=2.2 in]{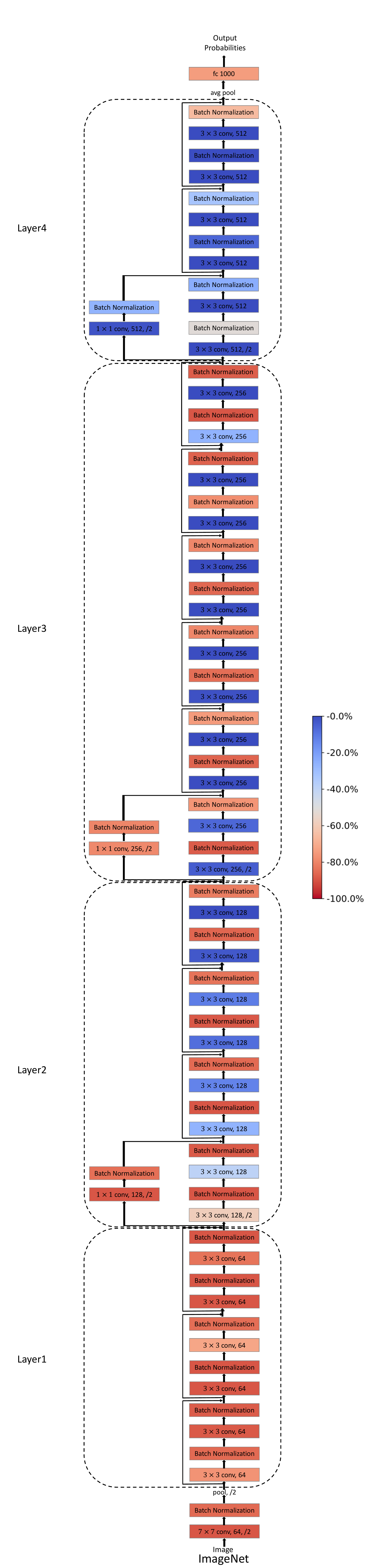}
\end{minipage}
}
\caption{Detailed visualization of the vulnerability of different components in ResNet-34 and Transformer ($n=100$, $L_{+\infty}$-norm). $\epsilon$ is set to $0.5$ for ResNet-34 (ImageNet) and $5$ for ResNet-34 (CIFAR-10). Warmer colors indicate significant performance degradation. We can see that 1) Lower layers in ResNet-34 are more robust to parameter corruption; 2) Batch normalization in ResNet-34 are usually less robust to
parameter corruption compared to its neighborhood components.}
\label{fig:resnet_details}
\end{figure*}

\end{document}